\documentclass[letterpaper,11pt]{article}
\usepackage{jmlr2e}
\usepackage{graphicx} 

\usepackage{savesym,multirow,url,xspace,graphicx,hyperref,enumitem,color}
\usepackage[titletoc]{appendix}

\usepackage{subcaption}
\usepackage{dsfont}
\usepackage{multicol}
\usepackage{wrapfig}
\usepackage{backref}

\usepackage{etoolbox}
\usepackage{enumitem}
\usepackage[utf8]{inputenc} 
\usepackage[T1]{fontenc}    
\usepackage{hyperref}       
\usepackage{url}            
\usepackage{booktabs}       
\usepackage{amsfonts}       
\usepackage{nicefrac}       
\usepackage{microtype}      
\usepackage{tikz}
\usepackage{pgfplots}
\usepackage{verbatim}
\usepackage{amsmath,bm}
\usepackage{cleveref}
\usepackage{multirow}
\usepackage{mathtools}
\usepackage{comment}

\usepackage[lined, boxed, ruled, commentsnumbered, noend]{algorithm2e}

\newtoggle{showcomments}
\settoggle{showcomments}{true}

\iftoggle{showcomments}{
\newcommand{\akash}[1]{{\textcolor{blue}{[{\bf A:} #1]}}}
\newcommand{\sanjoy}[1]{{\textcolor{red}{\rm [{\bf S:} #1]}}}
\newcommand{\annotate}[1]{{\textcolor{blue}{[{\bf Note:} #1]}}}
}
{
\newcommand{\akash}[1]{}
\newcommand{\sanjoy}[1]{}
\newcommand{\annotate}[1]{}
}

\makeatletter
\newif\if@restonecol
\makeatother

\usepackage{amssymb,xfrac}
\setitemize{noitemsep,topsep=1pt,parsep=1pt,partopsep=1pt, leftmargin=12pt}

\newtheorem{assumption}{Assumption}

\makeatletter

\newcommand{\Rmnum}[1]{\expandafter\@slowromancap\romannumeral #1@}
\makeatother
\usepackage{caption}

\newcommand{\M}{\mathbb{D}}

\numberwithin{equation}{section}

\newcommand{\defref}[1]{Definition~\ref{#1}}

\newcommand{\figref}[1]{Fig.~\ref{#1}}
\newcommand{\eqnref}[1]{\text{Eq.}~(\ref{#1})}
\newcommand{\secref}[1]{\textnormal{Section}~\ref{#1}}
\newcommand{\appref}[1]{Appendix \ref{#1}}

\newcommand{\thmref}[1]{Theorem~\ref{#1}}
\newcommand{\corref}[1]{Corollary~\ref{#1}}

\newcommand{\lemref}[1]{Lemma~\ref{#1}}

\newcommand{\assref}[1]{Assumption \ref{#1}}

\newcommand{\algoref}[1]{Algorithm~\ref{#1}}
\newcommand{\remref}[1]{Remark~\ref{#1}}

\newcommand{\paren} [1] {\ensuremath{ \left( {#1} \right) }}

\newcommand{\bracket}[1]{\left[#1\right]}

\newcommand{\curlybracket}[1]{\ensuremath{\left\{#1\right\}}}

\newcommand{\expctover}[2]{\mathop{\mathbb{E}}_{#1}\!\left[#2\right]}

\def \argmax {\mathop{\rm arg\,max}}
\def \argmin {\mathop{\rm arg\,min}}

\newcommand{\reals}{\ensuremath{\mathbb{R}}}

\newcommand{\sgn}[1]{\operatorname{sgn}\paren{#1}}

\newcommand{\cS}{{\mathcal{S}}}
\newcommand{\cR}{{\mathcal{R}}}

\newcommand{\cB}{{\mathcal{B}}}

\newcommand{\cM}{{\mathcal{M}}}
\newcommand{\cN}{{\mathcal{N}}}
\newcommand{\cX}{{\mathcal{X}}}

\newcommand{\cU}{{\mathcal{U}}}
\newcommand{\cT}{{\mathcal{T}}}

\newcommand{\cC}{{\mathcal{C}}}
\newcommand{\cO}{{\mathcal{O}}}
\newcommand{\cP}{{\mathcal{P}}}

\newcommand{\rank}[1]{\text{rank}({#1})}

\newcommand{\idot}{\boldsymbol{\cdot}}

\renewcommand{\tt}[1]{\textit{#1}}
\renewcommand{\sf}[1]{\textsf{#1}}
\def\BState{\State\hskip-\ALG@thistlm}

\newcommand{\maha}{\cM_{\textsf{maha}}}

\newcommand{\fin}{\cM_{\textsf{fin}}}
\newcommand{\nul}[1]{\mathrm{null}({#1})}
\newcommand{\symmp}{\text{Sym}_{+}(\reals^{p\times p})}
\newcommand{\symm}{\text{Sym}(\reals^{p\times p})}

\newcommand{\naive}{\cM_{\sf{NN}}}

\newcommand{\sA}{\mathrsfso{A}}

\newcommand{\sM}{\mathrsfso{M}}
\newcommand{\gm}{\mathchorus{e}}
\newcommand{\gM}{\mathchorus{E}}

\newcommand{\sH}{\sf{H}^*_c}
\newcommand{\sHx}{\sf{H}^*_x}
\newcommand{\sMx}{\sf{M}^*_x}
\newcommand{\sHc}[1]{\sf{H}^*_{c(#1)}}
\newcommand{\cd}{\cC}
\newcommand{\re}{\textcolor{black}{\omega}}

\newcommand{\ew}{h(\textcolor{purple}{\omega})}

\definecolor{cyanblue}{RGB}{170, 23, 205}
\newcommand{\tcg}[1]{\textcolor{cyanblue}{#1}}
\definecolor{redblue}{RGB}{10, 23, 205}

\newcommand{\ue}{U_{\sf{ext}}}
\newcommand{\g}{\kappa_0}
\pgfplotsset{compat=1.18}
\usepackage[utf8]{inputenc}
\usepackage[margin=1in]{geometry}
\usepackage{tikz-3dplot}
\usepackage{algpseudocode}
\usepackage{varwidth}
\usepackage{xcolor} 
\usepackage{empheq}
\usepackage{tikz}
\usepackage{framed}
\definecolor{shadecolor}{gray}{0.9}

\newcommand{\inner}[1]{\left<#1\right>}

\usepackage{booktabs}

\usepackage[margin=1in]{geometry}
\usepackage{xcolor} 
\usepackage{empheq}
\usepackage{framed}
\definecolor{shadecolor}{gray}{0.9}

\usepackage{booktabs}
\usepackage{natbib}
\definecolor{myblue}{RGB}{10, 112, 224}
\definecolor{bblue}{RGB}{100, 12, 224}
\hypersetup{hidelinks,colorlinks,citecolor=myblue}

\usepackage[T1]{fontenc}
\DeclareFontFamily{T1}{calligra}{}
\DeclareFontShape{T1}{calligra}{m}{n}{<->s*[1.44]callig15}{}
\DeclareMathAlphabet\mathcalligra   {T1}{calligra} {m} {n}
\DeclareMathAlphabet\mathzapf       {T1}{pzc} {mb} {it}
\DeclareMathAlphabet\mathchorus     {T1}{qzc} {m} {n}
\DeclareMathAlphabet\mathrsfso      {U}{rsfso}{m}{n}

\usetikzlibrary{calc, shadings} 
\usetikzlibrary{positioning,arrows.meta}
\newtoggle{longversion}
\settoggle{longversion}{true}

\usepackage{times}

\usepackage{verbatim}

\begin{document}

\title{Learning Smooth Distance Functions via Queries}

\author{\name Akash Kumar \email akk002@ucsd.edu \\
       \addr Department of Computer Science $\&$ Engineering,\\
	University of California, San Diego
        \AND
       \name Sanjoy Dasgupta \email sadasgupta@ucsd.edu \\
       \addr Department of Computer Science $\&$ Engineering,\\
	University of California, San Diego}


\maketitle

%

%


\begin{abstract}
 In this work, we investigate the problem of learning distance functions within the query-based learning framework, where a learner is able to pose triplet queries of the form: ``Is \(x_i\) closer to \(x_j\) or \(x_k\)?'' We establish formal guarantees on the query complexity required to learn smooth, but otherwise general, distance functions under two notions of approximation: \(\omega\)-additive approximation and \((1 + \omega)\)-multiplicative approximation. For the additive approximation, we propose a global method whose query complexity is quadratic in the size of a finite cover of the sample space. For the (stronger) multiplicative approximation, we introduce a method that combines global and local approaches, utilizing multiple Mahalanobis distance functions to capture local geometry. This method has a query complexity that scales quadratically with both the size of the cover and the ambient space dimension of the sample space.\end{abstract}

\section{Introduction}

In \emph{personalized retrieval} or, more generally, \emph{personalized prediction}, a machine learning system needs to adapt to a specific user's perceptions of similarity and dissimilarity over a space $\cX$ of objects (products, documents, movies, etc). Suppose the relevant structure is captured by a distance function $d: \cX \times \cX \to \mathbb{R}$. How can this be learned through interaction with the user?

This is one of the basic questions underlying the field of \emph{metric learning}~\citep{kulis_ml}. A key design choice in the interactive learning of distance functions is the nature of user feedback. One line of work assumes that the user is able to provide \emph{numerical dissimilarity values} $d(x,x')$ for pairs of objects $x,x'$~\citep{Cox2008,Stewart2005AbsoluteIB}; the goal is then to generalize from the provided examples to an entire distance function. Another option, arguably more natural from a human standpoint, is for the user to provide \emph{comparative} information about distances~\citep{distance_metric_relative, Jamieson2011LowdimensionalEU}. For instance, given a triplet of objects, $x,x',x''$, the user can specify which of $d(x,x')$ and $d(x,x'')$ is larger. This is the model we consider.

Earlier work in metric learning has focused primarily on learning Mahalanobis (linear transformations) distance functions~\citep{LMNN,Xing2002DistanceML,itml,NCA}. We look at the more general case of smooth but otherwise arbitrary distance functions $d$ over $\cX \subset \mathbb{R}^p$. We show a variety of results about the learnability of such distances using triplet queries.

The learning protocol operates as follows: on each round of interaction, 
\begin{itemize}
\item the learner selects a triplet $(x,x',x'') \in \cX^3$, and 
\item the user answers with
$\mbox{sign}(d(x,x') - d(x,x'')) \in \{-1, 0, 1\}$.
\end{itemize}
That is, the user either says that $d(x,x')$ and $d(x,x'')$ are equal, or if not, declares which is larger. After several such rounds, the learner halts and announces a distance function $d': \cX \times \cX \to \mathbb{R}$.

The best we could hope is for $d'$ and $d$ to agree on all triplets, that is, to be \emph{triplet-equivalent}. But this is unrealistic if the space of distance functions is infinite, since each query only returns a constant number of bits of information. We therefore consider two notions of \emph{approximate triplet-equivalence}.

The first, weaker, notion is that for any $x,x',x'' \in \cX$, 
\begin{gather*}
   |d(x,x') - d(x,x'')| > \omega \implies  \mbox{sign}(d(x,x') - d(x,x'')) = \mbox{sign}(d'(x,x') - d'(x,x'')). 
\end{gather*}
Here $\omega$ is a small constant that is pre-specified to the learner. In other words, $d$ and $d'$ agree on all triplets $(x,x',x'')$ for which there is a non-negligible gap between $d(x,x')$ and $d(x,x'')$. We show that this is achieved quite easily, by taking a finite cover of the space $\cX$ and selecting a distance $d'$ that agrees with $d$ on all triplets drawn from the cover.

\paragraph{Informal theorem 1.} Suppose $\cX \subset \mathbb{R}^p$. Let $N$ be the size of a $c\sqrt{\frac{\omega}{p^{1.5}}}$-cover of $\cX$ with respect to $\ell_2$ distance, where $c$ is a constant that depends upon the second-order smoothness of $d$. Then $O(N^2 \log N)$ queries are sufficient for learning a distance function that is $\omega$-approximately equivalent to $d$ in the (additive) sense above.

The problem with this \emph{additive} notion of approximation is that it provides no guarantees for triplets $(x,x',x'')$ in which there is a significant gap between $d(x,x')$ and $d(x,x'')$ in that the ratio over the distances $\max\{d(x,x')/d(x,x''), d(x,x'')/d(x,x')\}$ is large, even if the distances themselves are small. For this reason, a better and stronger notion of approximation is to require that $d$ and $d'$ agree on all triplets for which either $d(x,x') > (1+\omega) d(x,x'')$ or $d(x,x'') > (1+\omega) d(x,x')$.

Our main result is to realize this kind of multiplicative approximation. We use a hierarchical model consisting of a finite cover of the space $\cX$, and for each element of the cover, a Mahalanobis approximation to the distance function in the vicinity of that point. We show the following result.

\paragraph{Informal theorem 2.} Suppose $\cX \subset \mathbb{R}^p$. Let $N$ be the size of a $c(\frac{\omega}{p})^{1.5}$-cover of $\cX$ with respect to $\ell_2$ distance, where $c$ is a smoothness parameter for $d$. Then $O(N^2 \log N + N p^2 \log \frac{p}{\omega})$ queries suffice for learning a distance function that is $(1+\omega)$-approximately equivalent to $d$ in the multiplicative sense above.

A key component in this result is an efficient algorithm for learning a Mahalanobis distance function from triplet queries.

\section{Related Work}

There is a rich literature, spanning many decades, on methods for inferring the geometry of an unknown space, such as the internal semantic or preference space of a user. Early work dating back to the 1960s includes the psychometric literature on multidimensional scaling, which showed how to fit Euclidean distance functions to user-supplied dissimilarity values. More recently, it has become standard practice to train neural net embeddings using pairs of dissimilar and similar data points.

The geometry that is learned can be of three forms:
\begin{itemize}
\item A distance function, typically Mahalanobis distance.
\item A similarity function, typically a kernel function.
\item An embedding, nowadays typically a neural net.
\end{itemize}
The information used to learn these can consist of exact distance or similarity values, as in classical multidimensional scaling, or comparisons such as triplets and quadruples.
\paragraph{Learning a distance function}
Distance learning has been extensively studied in two main categories: (1) exact measurement information, where precise pairwise distances are provided (e.g., in metric multidimensional scaling) \citep{Cox2008, Stewart2005AbsoluteIB}, and (2) pairwise constraints informed by class labels, where the goal is to learn a distance function under the assumption that similar objects share the same class or cluster \citep{LMNN, Xing2002DistanceML, itml, NCA}. Further work has explored learning with ordinal or triplet constraints, where only relative comparisons between samples are provided \citep{distance_metric_relative, Kleindessner2016KernelFB}, and queries involving ordinal constraints \citep{Jamieson2011LowdimensionalEU}.

One specific class of distance functions, the \tt{Mahalanobis distance}, defined through a linear transformation, has received considerable attention in metric learning. In supervised contexts, studies such as \cite{LMNN}, \cite{Xing2002DistanceML}, \cite{itml}, and \cite{NCA} have focused on designing Mahalanobis distance functions that minimize intra-class distances while maximizing inter-class distances. This has been shown to improve prediction performance when integrated with hypothesis classes \citep{classifier_based, McFee2010MetricLT}. However, unlike these works, which rely on class supervision, our approach seeks to learn a target Mahalanobis or general distance function without any side information. In \tt{unsupervised metric learning}, classical methods like linear discriminant analysis (LDA) \citep{LDA} and principal component analysis (PCA) \citep{PCA} solve the problem by approximating a low-dimensional structure without requiring labeled data. Our work diverges from these methods, as we assume the learner queries triplets to obtain qualitative feedback, rather than having direct access to linearly transformed data.

Generalization bounds for metric learning in i.i.d. settings have also been explored for various loss functions \citep{Verma2015SampleCO, Mason2017LearningLM, Ye2019, sharp_gen, gen_bound_cao}. Of particular relevance to our work are the studies by \cite{distance_metric_relative} and \cite{Mason2017LearningLM}, which use triplet comparisons to learn Mahalanobis distance functions. Our approach differs from \cite{Mason2017LearningLM}, as we allow the learner to select triplets and receive exact labels, while they use i.i.d. triplets where the oracle may provide noisy labels. Additionally, \cite{pseudo_metric} addressed learning Mahalanobis distances in an online setting. For comprehensive surveys on Mahalanobis and general distance functions, refer to \cite{bellet_survey} and \cite{kulis_ml}.

\paragraph{Learning an embedding}
Learning appropriate representations or embeddings for a given dataset has been a central problem in machine learning, studied under various frameworks such as metric and non-metric multidimensional scaling, isometric embedding, and dimensionality reduction, among others. Classical approaches have addressed this problem in both linear \citep{PCA, Cox2008} and nonlinear settings \citep{Belkin2003LaplacianEF, isomap, Roweis2000NonlinearDR}, with the primary objective of uncovering a suitable low-dimensional representation of the data’s inherent structure. In contrast, our work focuses on using an embedding of the data space in the form of a cover that globally approximates the underlying latent distance function.

Previous research has extensively explored the concept of learning embeddings for objects that satisfy specific triplet or quadruplet constraints, broadly categorized into two main areas: ordinal embeddings \citep{Agarwal2007GeneralizedNM, AriasCastro2015SomeTF, Kleindessner2014UniquenessOO, Terada2014LocalOE} and triplet embeddings \citep{Jamieson2011LowdimensionalEU, Tamuz2011AdaptivelyLT, Maaten2012StochasticTE, Amid2015MultiviewTE, Heim2015EfficientOR, nowak_embed}. These studies primarily aim to develop methods for discovering embeddings that can be applied to various downstream tasks such as kernel learning, nearest neighbor search, and density estimation. In our work, we adopt a different framework, where the learner utilizes feedback from a user on triplet queries, leveraging this information to approximate a distance function that is accurate both locally and globally, up to a triplet similarity transformation.
\vspace{-4mm}




\section{Preliminaries}\label{sec: setup}

We denote by $\cX$ a space of objects/inputs. We consider a general notion of distance on this space.
\begin{definition}[distance function] A bivariate function \(d: \cX \times \cX \to \mathbb{R}\) is a distance function if 
$d(x,x') \geq 0$ and $d(x,x) = 0$ for all $x, x' \in \cX$.
\end{definition}
In this work, we are interested in various families of distance functions: distances on finite spaces; Mahalanobis distance functions; and smooth distances. For the latter two cases we will assume $\cX \subset \reals^p$.

\paragraph{Mahalanobis distance function}
Assume that \(\cX \subset \mathbb{R}^p\). A \textit{Mahalanobis} distance function is characterized by a symmetric, positive semidefinite matrix (written $M \succeq 0$, or alternatively \(M \in \symmp\)), and is denoted by \(d_M\), where for all $x, x' \in \cX$,
\begin{align}
     d_M(x, x') = \sqrt{(x - x')^{\top}M(x - x')} .\label{eq: maha}
\end{align}
We denote the family of Mahalanobis distance functions by
\begin{align*}
    \cM_{\mathsf{maha}} = \curlybracket{d_M: \cX \times \cX \to \reals_{\ge 0}\,|\, M \in \symmp}
\end{align*}
\looseness-1
\paragraph{Comparisons of distances} In this work, we obtain information about a target distance function on $\cX$ via triplet comparisons. For any triplet $x,y,z \in \cX$ and distance function $d$, we define the label $\ell((x,y,z);d) \in \curlybracket{-1,0,1}$ to be the sign of $d(x,y) - d(x,z)$. 


The label of a triplet corresponds to the three possibilities $d(x,y) < d(x, z)$, $d(x,y) = d(x,z)$, and $d(x,y) > d(x,z)$. From a human feedback standpoint, it would also be reasonable to allow the labels to be either $\leq$ or $\geq$. However, these relaxed labels would not be  sufficiently informative for our purposes; for instance, they are always satisfied by the trivial distance function $d(x,x') = 1(x \neq x')$.

\paragraph{Query learning a distance function:}  Given a family of distance functions $\cM$, the learner seeks to learn a target distance function $d$ within this space by querying the user/oracle. The learner is allowed rounds of interaction in which it adaptively pick triplets $(x,y,z) \in \cX^3$ and receives their labels.  We use the notation \( (x,y,z)_{\ell} \) to denote a triplet \( (x,y,z) \) along with its label \( \ell((x,y,z); d) \). We call this setup \tt{learning a distance function via queries}. 

Given the family $\cM$ we are interested in understanding the minimal number of triplet queries that the learner needs so that it can learn $d$ either exactly or approximately. To make this precise, we start by defining \emph{triplet equivalence} between two distance functions.
\begin{definition}[triplet equivalence]\label{def: triplet}
    On a given sample space $\cX$, we say distance functions $d$ and $d'$ are triplet-equivalent if for any triplet $(t,u,v) \in \cX^3$
    \begin{equation*}
        \ell((t,u,v); d) = \ell((t,u,v); d')
    \end{equation*}
\end{definition}
Triplet equivalence is achievable in finite instance spaces, but for more general spaces we need to allow approximations. We consider two approximate notions of triplet equivalence---additive (see \secref{sec: add}) and multiplicative (see \secref{sec: mult}). Formally, we define these approximations as
\begin{definition}[additive approximation] Given a target distance function $d$ and threshold $\omega > 0$, we say distance function $d'$ is triplet equivalent to $d$ up to additive factor $\omega$ if 
    \begin{align*}
        (\forall x,x',x'' \in \cX),\, |d(x,x') - d(x,x'')| > \omega \implies  \mbox{sign}(d(x,x') - d(x,x'')) = \mbox{sign}(d'(x,x') - d'(x,x''))
    \end{align*}
\end{definition}
\begin{definition}[multiplicative approximation] Given a target distance function $d$ and a threshold $\omega > 0$, 
we say distance function $d'$ is triplet equivalent to $d$ up to multiplicative factor $1+\omega$ if 
    \begin{align*}
        (\forall x,x',x'' \in \cX),\, d(x,x') > (1 + \omega) d(x,x'') \implies \mbox{sign}(d(x,x') - d(x,x'')) = \mbox{sign}(d'(x,x') - d'(x,x''))
    \end{align*}
\end{definition}


In this paper, we consider various families of distance functions: distances on finite spaces, Mahalanobis distances, and smooth distances. In each case, we study the query complexity of learning a target distance function upto approximate triplet equivalence, in both the additive and multiplicative sense.

\section{Additive approximation of a smooth distance function}\label{sec: add}
In this section, we show how to learn a smooth distance function $d$  up to approximate triplet equivalence in the $\omega$-additive sense.

The learner's strategy is to take an $\epsilon$-cover $\mathcal{C}$ of the instance space $\cX$ (with respect to $\ell_2$ distance) and to learn a distance function that is exactly triplet-equivalent to $d$ on $\mathcal{C}$. This distance function is then extended to all of $\cX$ using $\ell_2$ nearest neighbor. 

\subsection{Learning a distance function on a finite space, upto triplet equivalence}

In \thmref{thm: finite}, we demonstrate how any distance function on a finite space (e.g., $d$ on $\mathcal{C}$) can efficiently be learned upto exact triplet-equivalence. The full proof is deferred to \iftoggle{longversion}{\appref{app: finite}}{the supplemental materials}.
\begin{algorithm}[t]
\caption{Learning a smooth distance function up to additive approximation via $\ell_2$ covering}
\label{alg: naive_pseudo2}
\setlength{\algomargin}{3pt}
\LinesNumbered
\SetKwInOut{Given}{Given}
\SetKwInOut{Output}{Output}
\SetKwProg{For}{for}{}{}
\Given{Approximation threshold $\omega$.}
\Output{Distance function $d'$ on $\cX$.}
Construct an $\epsilon$-cover $\mathcal{C} \subset \cX$ with respect to $\ell_2$ distance, for $\epsilon$ as given in \thmref{thm: smoothl2}\;
Using triplet queries, construct distance function $\hat{d}$ on $\mathcal{C}$ that is triplet-equivalent to $d$ on $\mathcal{C}$, as in \thmref{thm: finite}\;
For any $x \in \cX$, set $c(x) = \arg\min_{c \in \mathcal{C}} \|x - c\|_2$\;
Define distance function $d'$ on $\cX$ by $d'(x,x') = \hat{d}(c(x),c(x'))$ \;
Output $d'$.
\end{algorithm}
\begin{theorem}\label{thm: finite}
    Given any distance function $d: \cX \times \cX \to \reals_{\ge 0}$ on a sample space $\cX$, and a finite subset $\cX_o \subset \cX$, a learner can find a distance function $\hat{d}: \cX_o \times \cX_o \to \reals_{\ge 0}$ that is triplet-equivalent to $d$ on $\cX_o$ using $|\cX_o|^2 \log |\cX_o|$ triplet queries.
\end{theorem}
\begin{proof}{\textbf{outline}:}
There are only $O(|\cX_o|^3)$ triplets involving points in $\cX_o$, so at most this many queries are needed. To get a better query complexity, we use a reduction to comparison-based sorting.

Fix any $x_o \in \cX_o$. In order to correctly answer all triplet queries of the form $(x_o, x', x'')$ for $x', x'' \in \cX_o$, it is both necessary and sufficient to know the \emph{ordering} of the distances $\{d(x_o, x): x \in \cX_o\}$. In other words, we need to sort $\cX_o \setminus \{x_o\}$ by increasing distance from $x_o$. We can do this by simulating a comparison-based sorting algorithm like mergesort. To compare which of $x', x''$ comes first in the ordering, we make the query $(x_o, x', x'')$. Mergesort asks for at most $|\cX_o| \log_2 |\cX_o|$ such comparisons, and thus we make at most this many triplet queries for $x_o$.

Repeating this for all $x_o \in \cX_o$, we get a total query complexity of $|\cX_o|^2 \log_2 |\cX_o|$. 

It is easy to devise a distance function $\hat{d}$ that is consistent with the resulting orderings, e.g. set $\hat{d}(x,x') = j$ if $x'$ is the $j$-th furthest point (in $\cX_o$) from $x$.
%
%
We provide a formal algorithm based on this strategy in \iftoggle{longversion}{\appref{app: finite}}{the supplemental materials} with a detailed proof of \thmref{thm: finite}. 
\end{proof}

\subsection{A notion of smooth distance function}

Let $\mathcal{C}$ be an $\epsilon$-cover of the instance space $\cX$, with respect to $\ell_2$ distance. \thmref{thm: finite} shows how to learn a distance function $\hat{d}$ on $\mathcal{C}$ that is exactly triplet-equivalent to the target distance function $d$ on $\mathcal{C}$. To extend $\hat{d}$ to all of $\cX$, we use $\ell_2$ nearest neighbor. For any $x \in \cX$, let $c(x)$ be its nearest neighbor in $\mathcal{C}$, that is, $\argmin_{c \in \mathcal{C}} \|x-c\|$. We will treat $c(x)$ as a stand-in for $x$.

This can be problematic if $d(x,c(x))$ is large even though $\|x - c(x)\| \leq \epsilon$. We now introduce a smoothness condition that precludes this possibility.

\begin{definition}[$(\alpha,L,\delta)$-smooth]
    Suppose $\cX \subset \reals^p$. We say a distance function $d: \cX \times \cX \to \reals_{\ge 0}$ is $(\alpha,L,\delta)$-smooth for $\alpha,L,\delta > 0$ if for any $x,x' \in \cX$ such that $||x - x'||_2 \le \delta$ we have 
    \begin{align*}
        d(x, x') \le L\cdot ||x - x'||_2^{\alpha}
    \end{align*}
\end{definition}
We'll see that this smoothness condition holds for distance functions with bounded third derivatives.
\begin{assumption}[local smoothness]\label{assump: a1}
Suppose \(\cX \subset \mathbb{R}^p\). For \(x \in \cX\), define \(f_x: \cX \to \mathbb{R}\) by \(f_x(x') = d(x, x')\).
Assume that for all $x \in \cX$, the function \(f_x\) is \(C^3\) in an open ball around \(x\) and that its third partial derivatives are bounded in absolute value by a constant $M > 0$. Let \(\sHx\) be the Hessian of \(f_x\) at \(x\). 
\end{assumption}

The following result follows from the Taylor expansion of $f_x$. We will make heavy use of it.

\begin{lemma}[Taylor's theorem]\label{lem: taylorlemma} Fix $\cX \subset \reals^p$. If distance function $d: \cX \times \cX \to \reals_{\ge 0}$ satisfies Assumption~\ref{assump: a1}, 
then for any $x,x' \in \cX$, we have 
\begin{align*}
    \bigg\lvert d(x,x') - \frac{1}{2}(x'-x)^{\top} \sHx (x'-x) \bigg\rvert \le  \frac{Mp^{\frac{3}{2}}}{6} ||x-x'||_2^3, 
\end{align*}
where $\sHx$ denotes the Hessian of the function $d(x,\cdot)$ at $x$. Furthermore, $\sHx$ is a symmetric, positive semidefinite matrix, i.e. $\sHx \succeq 0$. 
\end{lemma}

This lemma suggests that for a point $x \in \cX$, the distance function in the vicinity of $x$ is well-approximated by the Mahalanobis distance 
$$ d(x,x') \approx \frac{1}{2} (x'-x)^T \sHx (x'-x) .$$
Let's take this a step further, by considering
\begin{align*}
    (x'-x)^{\top} \sHx (x'-x) \le \lambda_{\max}\paren{\sHx}\cdot ||x-x'||_2^2 \label{eq: mahaapp}
\end{align*}
where $\lambda_{\max}(\sHx)$ denotes the largest eigenvalue of $\sHx$. If these eigenvalues are uniformly upper-bounded by a constant $\gamma > 0$, then it follows from Lemma~\ref{lem: taylorlemma} that the distance function $d$ is $(\alpha, L, \delta)$-smooth for $\alpha = 2$, $\delta = 1$, and $L = (\gamma/2) + (Mp^{3/2}/6)$.

\subsection{Learning a smooth distance upto additive approximation}


Under $(\alpha, L, \delta)$-smoothness of the target distance $d$, it is enough to learn a distance function $\hat{d}$ that is triplet-equivalent to $d$ on a finite cover $\mathcal{C}$ of $\cX$, and then extend it to the rest of the space using $\ell_2$ nearest neighbor. See \algoref{alg: naive_pseudo2}.

\paragraph{Distance Model:} Consider a distance function model $\naive = \curlybracket{(d, (\cC, d_{\cC}), \ell_2)}$ where every distance function $d \in \naive$ is parameterized by a finite subset $\cC \subset \cX$ and a distance function $d_{\cC}: \cC \times \cC \to \reals_{\ge 0}$. Then $d: \cX \times \cX \to \reals_{\ge 0}$ has the following form:
\begin{equation*}
    d(x,y) = d_{\cC}(c(x), c(y)) \text{ where } c(x) = \argmin_{c \in \cC} ||c - x||_2.\vspace{-2mm}
\end{equation*}

For this family of distance functions, we show that a \(c \omega^{\frac{1}{\alpha}}\)-cover of the space is sufficient to learn an \((\alpha, L, \delta)\)-smooth \(d\), where \(c\) depends on the smoothness parameters of the distance function.
To achieve the formal guarantee, we also require a directional form of the triangle inequality in addition to \((\alpha, L, \delta)\)-smoothness, which is formalized in the following.

\begin{assumption}[triangle inequality]\label{assump: a0}
    Suppose $\cX \subset \reals^p$. We say a distance function $d$ defined on $\cX$ satisfies triangle inequality\footnote{This doesn't require symmetry of $d$.} over $\cX$ if for all $x,y,z \in \cX$, $d(x,y) + d(y,z) \ge d(x,z)$.
\end{assumption}
\begin{remark}\label{rem: relaxation} The requirement of this directional triangle inequality is for the ease of analysis. With a minor modification to the triangle inequality, such as:
$$
\forall x, y, y' \in \mathcal{X}, \quad |d(x, y) - d(x, y')| \leq C \cdot \||y - y'\||_2^{\alpha},
$$
for constants $C, \alpha > 0$, we can lift the requirement of \assref{assump: a0}. As discussed in the later sections, this allows this work to include KL divergence in the analysis.
\end{remark}

Now, we state the main result of the section with the proof deferred to \iftoggle{longversion}{\appref{app: l2}}{the supplemental materials}.

\begin{theorem}\label{thm: smoothl2} Consider a compact subset $\cX \subset \reals^p$. Consider a distance function $d: \cX \times \cX \to \reals_{\ge }$ that is $(\alpha, L,1)$-smooth for $\alpha, L > 0$ and satisfies \assref{assump: a0}. Then, \algoref{alg: naive_pseudo2} outputs a distance function $d' \in \naive$ triplet equivalent to $d$ up to $\omega$-additive approximation. 
The query complexity to find such a $d'$ is $\paren{\cN(\cX, \epsilon, \ell_2)^2 \log \cN(\cX, \epsilon, \ell_2)}$ triplet queries where $\epsilon = \min(1, (\omega/4L)^{1/\alpha})$ and $\cN(\cX, \epsilon, \ell_2)$ is the covering number of the space $\cX$ with $\epsilon$-balls in $\ell_2$ norm.
\end{theorem}

\begin{proof}{\textbf{outline}:}
As outlined in \algoref{alg: naive_pseudo2}, the query strategy involves constructing an $\epsilon$-cover of the space, denoted by \(\mathcal{C} \subset \mathcal{X}\). For any samples \(x, y, z \in \mathcal{X}\), the nearest neighbors \(c(x), c(y), c(z)\) in \(\mathcal{C}\) are determined. Leveraging the smoothness property, the distances of the form \(d(x, c(x))\)  and \(d(c(x), x)\) (even though $d$ need not be symmetric, local smoothness guarantees a bound) can be bounded by \(L \cdot \epsilon^{\alpha}\), since \(\|x - c(x)\|_2 \leq \epsilon\). With the chosen cover radius and the triangle inequality, it follows that \(|d(x, y) - d(c(x), c(y))| \leq 2L \cdot \epsilon^{\alpha}\). Finally, we need to ensure that the error bound of \(2L \cdot \epsilon^{\alpha}\) is a constant factor smaller than \(\omega\), which is guaranteed by the choice of $\epsilon$ specified in the theorem.
\end{proof}
The results in \thmref{thm: smoothl2} can be specialized to the case of distances that satisfy 
\assref{assump: a1}, using the remarks that follow that assumption.

\begin{corollary}\label{cor: smoothl2}
Consider a compact subset \(\mathcal{X} \subset \mathbb{R}^p\) and a distance function \(d: \mathcal{X} \times \mathcal{X} \to \mathbb{R}_{\geq 0}\) that satisfies \assref{assump: a1} and \assref{assump: a0}. Furthermore,  suppose that for all \(x \in \mathcal{X}\), $\lambda_{\max}(\sHx)$ is upper bounded by a constant \(\gamma > 0\).
Then, \algoref{alg: naive_pseudo2} outputs a distance function \(d' \in \naive\) that is triplet-equivalent to \(d\) up to an \(\omega\)-additive approximation, using at most \(\mathcal{N}(\mathcal{X}, \epsilon, \ell_2)^2 \log \mathcal{N}(\mathcal{X}, \epsilon, \ell_2)\) triplet queries, where $\epsilon \le \paren{\frac{\omega'}{2\gamma + 4K}}^{\frac{1}{2}}$ for $w' := \min\{1,\omega\}$ and $K := \frac{M p^{1.5}}{6}$.
\end{corollary}
\section{Multiplicative approximation of a smooth distance function}\label{sec: mult}
Here, we aim to learn a target distance \(d\) up to multiplicative approximation. This is a stronger notion than additive approximation, where the learner only needs to correctly distinguish distances between points \(x, x', x'' \in \mathcal{X}\) such that the distances  differ by at least \(\omega\); in particular, small distances can be ignored.
In contrast, in the multiplicative approximation case, the learner must learn to distinguish distances that are arbitrarily small, if they differ by a multiplicative factor of \( (1+\omega) \), for any arbitrary but fixed choice of \( \omega > 0 \).

An intuition from the previous section is that if the distances are indeed large, then a global approximation based on a cover with a suitable choice of \( \ell_2 \) norm radius should suffice. However, the challenge remains in how to \tt{approximate small local distances} for arbitrary distance functions. 
A natural approach is to approximate using local Mahalanobis distance functions. 
Indeed, \lemref{lem: taylorlemma} indicates that under mild conditions, locally, \( d \) behaves as a Mahalanobis distance function up to some additive error. We show that this local linear approximation can be used to obtain $\omega$-multiplicative approximation.

In this pursuit, we first demonstrate how to query learn local Mahalanobis distances with quadratic dependence on the ambient space dimension.
\subsection{Learning a Mahalanobis distance function via queries}\label{sec: maha}
Although prior work has studied the problem of learning Mahalanobis distances with triplet and pairwise comparisons in the i.i.d setting~\citep{wang2024metric,Mason2017LearningLM, distance_metric_relative} and adaptive embedding learning with pairwise comparisons~\citep{Jamieson2011LowdimensionalEU}, the query complexity in the query framework is underexplored. 
We show that a binary search-based method, along a predefined set of directions in $\reals^p$, can statistically efficiently approximates a target matrix in $\maha$ up to a linear scaling, discussed concretely in \algoref{alg:learnapproxmaha}.



First note that positive semidefinite matrices $M$ and $cM$, for $c > 0$, yield Mahalanobis distance functions $d_M$ and $d_{cM}$ that are triplet-equivalent. Thus we can only hope to recover a target matrix $M^*$ upto a scaling factor. Moreover, since each query yields a constant number of bits of information, the recovery will necessarily be approximate. We achieve approximation in Frobenius norm: for any desired $\epsilon > 0$, we use triplet queries to find a matrix $M \succeq 0$ such that $\|M - \tau M^*\| \leq \epsilon$, where the constant $\tau$ is chosen so that the largest diagonal element of $\tau M^*$ is 1.


In particular, we first identify the coordinate vector $y \in \{e_1, \ldots, e_p\}$ for which $y^T M^* y$ is maximized; this requires $p$ triplet queries of the form $(0, e_i, e_j)$. We then use queries to obtain $p(p+1)/2-1$ other (approximate) linear constraints to return a matrix $M$ such that $y^{\top}My = 1$ and solve the resulting system to get the distance matrix.

We state the claim of the result in \thmref{thm: mahamain} with a detailed proof deferred to \iftoggle{longversion}{\appref{app: maha}}{the supplemental materials}. In what follows, $\kappa(M^*)$ is the condition number of matrix $M^*$ and we assume it has been normalized so that its largest diagonal entry is 1.
\begin{theorem}\label{thm: mahamain}
    Fix an input space $\cX \subset \reals^p$, an error threshold $\epsilon > 0$, and a target matrix $M^* \in \maha$. Then, \algoref{alg:learnapproxmaha} run with $\epsilon_{\sf{alg}} = \frac{\epsilon}{2p^2}$ outputs a distance function $d_{M} \in \maha$ such that
    \begin{align*}
        ||\tau M^* - M||_{F} \le \epsilon 
    \end{align*}
    using $\frac{p(p+1)}{2} \log \paren{\frac{2p^2 \kappa(M^*)^2}{\epsilon}} + p$ triplet queries, where $\tau$ is a scaling factor so that $\max_i (\tau M^*)_{ii} = 1$. 
    \vspace{-2.mm}
\end{theorem}

\begin{proof}{\textbf{outline}:} 
There are \(\frac{p(p+1)}{2}\) degrees of freedom for a given symmetric matrix. If these degrees of freedom are correctly approximated, it becomes possible to learn the matrix within a controlled error. This is the central idea of the proof of the theorem, where we construct a basis \(\mathcal{B}\) of rank-1 matrices $\{u_i u_i^T\}$ for the space \(\symm\). One of these $u_i$'s is the coordinate vector $y$ described above.
For each remaining $u_i$, we use binary search and triplet queries to approximately find the value $c_i^* = (u_i^T M^* u_i)/(y^T M^* y)$; in particular, we find $\hat{c}_i = c_i^* \pm O(\epsilon)$ using $\log \kappa(M^*)^2/\epsilon$ queries. We then solve the linear system $u_i^T M u_i = \hat{c}_i y^T M y$ to get $M$ and project the result to the semidefinite cone\footnote{PSD cone is at most distance $\epsilon$ (in Frobenius norm) from $M$ since $M^*$ lies in it.}. To bound the Frobenius norm between $M^*$ and $M$, we analyze the condition number of the linear system of equations.
\end{proof}
\begin{algorithm}[t]
\caption{Learning a Mahalanobis distance function with triplet queries}
\label{alg:learnapproxmaha}

\SetKwInOut{Input}{Input}
\SetKwInOut{Output}{Output}
\LinesNumbered
\Input{Input space $\mathcal{X}$, distance function model $\maha$, error threshold $\epsilon$}
\Output{Mahalanobis distance function $d'$ on $\cX$}

\text{Initialize}: $\mathcal{T} = \emptyset$\;

\SetKwProg{Fn}{Function}{}{end}
\SetKwProg{ForEach}{foreach}{}{}

    Set $U := \curlybracket{e_i}_{i \in \bracket{p}}$ to the standard basis of $\reals^p$\;
    Learner triplet queries to find $ y \in U$ such that $y^\top M^* y = \max_{u \in U} u^\top M^* u$\; \tcp{\footnotesize \textcolor{bblue}{Queries of the form $\{(0, u,u'): u \neq u' \in U\} $}}
    Set $\ue := U \cup \curlybracket{\frac{(e_i + e_j)}{\sqrt{2}} : 1 \le i < j \le p}$\;
    \ForEach{$i \in \{1, 2, \ldots, \frac{p(p+1)}{2}\}$}{
        \text{Set} $u_i := \ue[i]$\; 
        \text{Set} $\hat{c}_i \gets \text{BinarySearch}(y, u_i, \epsilon)$\; \tcp{\footnotesize\textcolor{bblue}{This finds $\hat{c_i}$ that is $\epsilon$-close to $c_i^*$ defined as $u_i^\top M^*u_i = c_i^* y^\top M^* y$}}
        Update $\mathcal{T} \gets \mathcal{T} \cup \{(0, \sqrt{\hat{c}_i}y, u_i)_{0}\}$\;
        \tcp{\footnotesize\textcolor{bblue}{Append $(0, \sqrt{\hat{c}_i}y, u_i)$ with label 0}}
    }
Learner selects a Mahalanobis distance function $d' \in \maha$ according to the linear constraints $\cT$\; \tcp{\footnotesize\textcolor{bblue}{Solve for all $u_i \in \ue$ $u_i^{\top} M u_i = \hat{c_i}$ in the space of $p \times p$ symmetric matrices and then project onto the semidefinite cone $\symmp$.}}
Output $d'$\;

\end{algorithm}

\paragraph{Query Learning a Local Mahalanobis Distance with a Smooth Distance Function:} 

\thmref{thm: mahamain} assumes that triplet queries are answered with respect to a target Mahalanobis distance $d_{M^*}$. However, we will be using Mahalanobis distance to locally approximate a smooth distance function in the vicinity of particular points $x$. In that case, the Mahalanobis matrix will be the Hessian $\sHx$, as indicated in \lemref{lem: taylorlemma}, but this only \emph{approximates} the target distance function, and so we cannot get exact answers to queries about $d_{\sHx}$. 

Looking at \algoref{alg:learnapproxmaha}, we see that it makes triplet queries of the form $(0, u, v)$; we will answer them by passing triplet queries of the form $(x, x+\rho u, x + \rho v)$ to the target distance function $d$, for $\rho > 0$. It turns out that by setting $\rho$ sufficiently small, the approximation error in using the target distance function $d$ instead of the Mahalanobis distance $d_{\sHx}$ can be controlled.

In \thmref{thm: advmahamain}, we provide the formal guarantee under the following assumption.
\begin{assumption}[bounded eigenvalues]\label{assump: a2} There exists scalar $\gM, \gm > 0$ such that \(\gM I_p \succeq \sHx \succeq \gm I_p\) $\forall x \in \cX$.
\end{assumption}

The proof of \thmref{thm: advmahamain} appears in \iftoggle{longversion}{\appref{subapp: advmahamain}}{the supplemental materials}. In \algoref{alg: learnapproxmahamain} in the appendix, we provide the modified version of \algoref{alg:learnapproxmaha} which incorporates the scaled and shifted  triplet queries of the form $(x, x+\rho u, x + \rho v)$.
\begin{theorem}\label{thm: advmahamain}
    Let \(\mathcal{X} \subset \mathbb{R}^p\) be a subset. Consider a distance function $d: \cX \times \cX \to \reals_{\ge 0}$ that satisfies \assref{assump: a1} and \assref{assump: a2} with scalars $\gm, \gM, M > 0$. Fix a sample $x \in \cX$ and an error threshold $\epsilon \in \paren{0, \frac{3\gm^3}{2M p^{1.5} \gM^2}}$. Then, (modified) \algoref{alg:learnapproxmaha} run with $\epsilon_{\sf{alg}} = \frac{\epsilon}{2p^2}$ outputs a distance function $d_{\sf{H}_x} \in \maha$ within $\frac{p(p+1)}{2}\log \left( \frac{2 p^2 \gM^2 }{\gm^2 \epsilon}\right)  + p$ triplet queries such that\vspace{-2mm}
    \begin{align*}
        ||\tau_x\sHx - \sf{H}_x||_{F} \le \epsilon
    \end{align*}
    where $\tau_x \in \bracket{{1}/{\gM}, {1}/{\gm}}$.
\end{theorem}
\begin{remark}\label{rem: errorbound} The bound on the error threshold $\epsilon$ in \thmref{thm: advmahamain} is due to practical considerations of the \algoref{alg: learnapproxmahamain}. 
    This requirement pertains to a scenario where the learner doesn't have a good estimate of a bound on the scaling factor $\rho$ which is a function of the other constants $\gm, \gM, M$, and $p$ (see \appref{subapp: advmahamain}).
\end{remark}
\vspace{-5mm}


\subsection{(1 + $\omega$)-multiplicative approximation via local Mahalanobis distance functions}

\begin{algorithm}[t]
\caption{Learning a smooth distance function via local Mahalanobis distance functions}
\label{alg: smoothdisthessian}
\SetKwProg{ForEach}{foreach}{}{}
\LinesNumbered
\SetKwInOut{Given}{Given}
\SetKwInOut{Output}{Output}
\Given{Approximation threshold $\omega$.}
\Output{Distance function $d'$ on $\cX$.}
Construct an $\epsilon$-cover $\mathcal{C} \subset \cX$ with respect to $\ell_2$ distance, for $\epsilon$ as given in \thmref{thm: smoothhessian}\;
Using triplet queries, construct distance function $d_f$ on $\mathcal{C}$ that is triplet-equivalent to $d$\;
Learner triplet queries to learn local Mahalanobis distance functions $\curlybracket{\sf{H}_c}_{c \in \cC}$ with error threshold $\xi$ as shown in \thmref{thm: smoothhessian}\;
Learner picks $\theta := 4\hat{\beta}$ as shown in \thmref{thm: smoothhessian}\;
For any $x \in \cX$, set $c(x) := \argmin_{c \in \mathcal{C}} \|x - c\|_2$\;
\ForEach{ $(x,y) \in \mathcal{X}^2$ to compute $d'(x, y)$}{
    \If{$(y - x)^\top \sf{H}_{c(x)} (y - x) > \theta$}{
         $d'(x, y) = d_{f}(c(x), c(y)) + \theta$\;
    }
    \Else{
         $d'(x, y) = (y - x)^\top \sf{H}_{c(x)} (y - x)$\;
    }
}
Output $d'$.
\end{algorithm}

For a \((1 + \omega)\)-multiplicative approximation, we integrate global and local strategies: a cover of the space combined with local Mahalanobis distance functions at the centers of the cover. \algoref{alg: naive_pseudo2} describes learning a smooth distance function up to triplet equivalence using a finite distance function, while \algoref{alg:learnapproxmaha} approximates local Mahalanobis distances, as shown in \thmref{thm: advmahamain}. These approximations are combined through a switch determined by a threshold that is a function of \(\omega\). The query strategy is outlined in \algoref{alg: smoothdisthessian}. For any pair \(x, y \in \mathcal{X}\), a local Mahalanobis distance \((y - x)^\top \mathsf{H}_{c(x)} (y - x)\) is computed to decide whether to use the global or local distance function, with the threshold depending on \(\omega\). This achieves the \((1 + \omega)\)-multiplicative approximation of a smooth distance function, as formally proven in \thmref{thm: smoothhessian}. 

In this section, we consider a distance model that combines local and global distances as follows:
\paragraph{Distance Model:} 
Consider a distance function model defined for a threshold \(\theta\), denoted as \(\widehat{\mathcal{M}}(\theta) = \{(d, (\mathcal{C}, d_{\mathcal{C}}), (\mathcal{U}_c, d_{M_c})_{c \in \mathcal{C}}, \ell_2)\}\), where each distance function \(d \in \widehat{\mathcal{M}}(\theta)\) is parameterized by a finite set of centers \(\mathcal{C} \subset \mathcal{X}\), along with a distance function \(d_{\mathcal{C}}\) defined on \(\mathcal{C}\). For each center \(c \in \mathcal{C}\), there is a local Mahalanobis distance function \(d_{M_c}\), defined on its neighborhood \(\mathcal{U}_c \subset \mathcal{X}\), with \(M_c \succeq 0\). The overall distance function \(d\) takes the form:
\[
d(x, y) = 
\begin{cases} 
d_{\mathcal{C}}(c(x), c(y)) + \theta, & \text{if } d_{M_{c(x)}}(x, y) > \theta, \\
d_{M_{c(x)}}(x, y), & \text{otherwise},
\end{cases}
\]
where \(c(x)\) is the center such that \(x \in \mathcal{U}_c\) (ties are broken with smallest $\ell_2$ distance). The goal is to approximate a smooth distance function up to a \((1 + \omega)\)-multiplicative factor within \(\widehat{\mathcal{M}}(\theta)\).

Before stating the main result, we discuss two regularity conditions on the curvature and arbitrarily small distances of a distance function.
\begin{assumption}[Hessian continuity]\label{assump: a4} Consider a distance function $d: \cX \times \cX \to \reals_{\ge}$ that is $C^2$ in its second argument. We say $d$ is Hessian continuous if there exists a scalar $L > 0$ such that 
\begin{align*}
    \forall x, x' \in \cX,\quad ||\sHx - \sf{H}^*_{x'}||_{F} \le L\cdot ||x - x'||_2.
\end{align*}
\end{assumption}
Note that the assumption above stipulates the smoothness of the Hessians around a sample. More concretely, it asserts smoothness in the first argument of a distance function which does not follow from the \assref{assump: a1} which only asserts smoothness in the first argument with bounded partial derivatives.

An immediate consequence of this result is that if $x$ and $x'$ are close to each other then the distances wrt sample $x'$ in the neighborhood of $x$ can be computed using the Hessian information at $x$ (see \iftoggle{longversion}{\appref{app: hessian}}{the supplemental materials} for further discussion). 

For a general distance function \(d\), it is possible that samples far apart in \(\ell_2\) distance are infinitesimally close (or even 0) under \(d\), which renders both global and local approximation schemes ineffective. For example, consider a distance function $d(x,y) = ||x-y||_2^2\cdot e^{-\tan ||x-y||_2}$, where $\cX = \pi/2\cdot \mathbb{S}_2$. In this case, as $||x-y||_2 \to \pi/2$ we get $d(x,y) \to 0$. Now, no bounded choice of $\epsilon > 0$ for an $\epsilon$-cover can even yield an $\omega$-additive approximation for small enough $\omega > 0$.

To avoid such scenarios, we impose the following condition on the distance function.
\begin{assumption}[non-zero distances]\label{assump: a5}
    For any \(\delta > 0\), \(\Delta_\delta := \liminf_{x, y \in \mathcal{X}} \{d(x, y) : \|x - y\|_2 > \delta\}\) is strictly positive, i.e., $\Delta_\delta > 0$.
\end{assumption}

Now, we state the main claim of the section with the proof detailed in \iftoggle{longversion}{\appref{app: hessian}}{the supplemental materials}.

\begin{theorem}\label{thm: smoothhessian}
Let \(\mathcal{X} \subset \mathbb{R}^p\) be a compact subset. Consider a distance function $d: \cX \times \cX \to \reals_{\ge 0}$ that satisfies \assref{assump: a1} (locally smooth), \assref{assump: a0} (triangle inequality), \assref{assump: a2} (bounded eigenvalues), \assref{assump: a4} (Hessian continuity), and \assref{assump: a5} (non-zero distances) for constants $\gm,\gM,M, L > 0$. 

Then, \algoref{alg: smoothdisthessian} outputs a distance function \( d' \in \widehat{\mathcal{M}}(\theta) \), triplet-equivalent to \( d \) up to \(\omega\)-multiplicative approximation, with a query complexity of at most \( N^2 \log N + N p^2 \log \frac{1}{\vartheta} \), where \( N = \mathcal{N}(\mathcal{X}, \epsilon, \ell_2) \), \( \epsilon \leq \frac{C \cdot w_1}{p^{\frac{3}{2}}} \) is the covering radius, and \( \xi \leq {D \cdot w_2} \) is the error threshold for local Mahalanobis distance functions. The parameters \( w_1, w_2, \theta, \vartheta \) are bounded as:
\begin{itemize}
    \item \( w_1 := \omega^\frac{3}{2}\cdot\mathbf{1}_{\{\omega < 1\}} + \frac{C_1}{\sqrt{\omega}}\cdot\mathbf{1}_{\{\omega \geq 1\}} \)
    \item \( w_2 := \omega\cdot\mathbf{1}_{\{\omega < 1\}} + C_2\cdot\mathbf{1}_{\{\omega \geq 1\}} \)
    \item \( \theta \leq \frac{4C' \cdot w_3}{p^3}, \quad w_3 := \omega^2\cdot\mathbf{1}_{\{\omega < 1\}} + \frac{C_3}{\omega}\cdot\mathbf{1}_{\{\omega \geq 1\}} \)
    \item \( \vartheta \leq \frac{C'' \cdot w_4}{p^{2}}, \quad w_4 := \omega\cdot\mathbf{1}_{\{\omega < 1\}} + C_4\cdot\mathbf{1}_{\{\omega \geq 1\}} \)
\end{itemize}
with constants \( D, C, C', C'', C_1, C_2, C_3, C_4 \) dependent on \( \gm, \gM, \Delta_\delta, M, L \), with \( \delta \le \frac{3\gm}{2M p^{\frac{3}{2}}} \).
\end{theorem}
\begin{proof}{\textbf{outline}:}
The key to the proof is to carefully describe the local behavior of the distance function \(d\) via local Mahalanobis distance functions, allowing for the computation of small distances with bounded error. Using \assref{assump: a1}, we show that in a small neighborhood around any sample, \(d\) is both \tt{strongly convex} and \tt{smooth}, meaning that its growth rate can be measured and depends on the eigenvalue bounds \(\gm\) and \(\gM\). 
We use this information to design a threshold for switching between global distance function on a cover and local Mahalanobis distance functions.

The size of the neighborhood where local approximation with small error is feasible depends on how small \(d(x, \cdot)\) can be around a given sample \(x \in \mathcal{X}\) and the local curvature of \(d\). As stated in \assref{assump: a5}, \(\min_{x \in \mathcal{X}} \{d(x, x') : x' \in \mathcal{X} \setminus B_2(x, \delta)\}\) is lower-bounded by a positive constant for any fixed \(\delta > 0\). The key idea is to select \(\delta\) within the zone of strong convexity and smoothness, leading to the result that, for a sufficiently small \(\hat{\beta}\) and samples $x,y \in \cX$, if \((y - x)^{\top}\sf{H}_{c(x)}(y - x) > 4\hat{\beta}\), then \(d(x,y)\) is also lower-bounded by \(c'\hat{\beta}\), where \(c'\) depends on the smoothness parameters. This guides when to switch between the global distance function and local Mahalanobis distance functions. Since the global distance function is triplet-equivalent to \(d\), reducing samples to their centers for large distances ensures the multiplicative approximation holds.


For distances within the \(\delta\)-ball, \algoref{alg: smoothdisthessian} suggests using the uncentered distance wrt the Hessian at the closest center to the sample of interest. This approach is reliable because the Hessian continuity, as stated in \assref{assump: a4}, ensures that the local distance computation remains accurate (up to small controllable error) within the neighborhood. 
\end{proof}
\begin{remark}
    As indicated in \remref{rem: relaxation}, with a minor modification to \assref{assump: a0}, \thmref{thm: smoothhessian} can be stated without the requirement of the triangle inequality; thus \thmref{thm: smoothhessian} applies to a wide class of Bregman divergences studied in the machine learning literature, including KL divergence (which violates both symmetry and triangle inequality). 
\end{remark}

\nocite*
\bibliography{ref}
\newpage




 \iftoggle{longversion}{
\newpage



\appendix
\section{Learning Smooth Distance Functions via Queries: \\
Supplementary Materials}
\vspace*{10mm}
\begin{itemize}
    \item \appref{app: finite}: \textbf{Learning a distance function on finite sample}\vspace{2mm}
    \begin{itemize}
        \item Proof of \thmref{thm: finite} for finite-sample distance functions.\vspace{1mm}
    \end{itemize}
    
    \item \appref{app: l2}: \textbf{Learning a smooth distance function via $\ell_2$ covering of the space}\vspace{2mm}
    \begin{itemize}
        \item \appref{subapp: taylor}: Taylor's approximation and proof of \lemref{lem: taylorlemma}.\vspace{1mm}
        \item \appref{subapp: theoreml2}: Proof of \thmref{thm: smoothl2}.\vspace{1mm}
        \item \appref{subapp: corollaryl2}: Proof of \corref{cor: smoothl2}.\vspace{1mm}
    \end{itemize}
    
    \item \appref{app: maha}: \textbf{Learning a Mahalanobis distance function}\vspace{2mm}
    \begin{itemize}
        \item \appref{subapp: condition}: Construction of a basis of $\symm$ with lower-bounded eigenvalue.\vspace{1mm}
        \item \appref{subapp: mahamain}: Proof of \thmref{thm: mahamain}.\vspace{1mm}
        \item \appref{subapp: advmahamain}: Proof of \thmref{thm: advmahamain}.\vspace{1mm}
        \item \figref{fig:extendedalgos}: Extended procedure for \algoref{alg:learnapproxmaha}.\vspace{1mm}
        \begin{itemize}
            \item Learning Mahalanobis distance function (noiseless setting in \thmref{thm: mahamain}): \algoref{alg: learnapproxmahamain1}.\vspace{1mm}
            \item Learning Mahalanobis distance function (noisy setting in \thmref{thm: advmahamain}): \algoref{alg: learnapproxmahamain}.\vspace{1mm}
        \end{itemize}
    \end{itemize}
    
    \item \appref{app: hessian}: \textbf{Learning a smooth distance function via local Mahalanobis distance functions}\vspace{2mm}
    \begin{itemize}
        \item \appref{subapp: close}: Close approximation of a Hessian matrix: \lemref{lem: taubound}.\vspace{1mm}
        \item \appref{subapp: nice}: Well-behaved properties of a smooth distance function: \lemref{lem: convex}, \lemref{lem: smalldist}.\vspace{1mm}
        \item \appref{subapp: final}: Proof of \thmref{thm: smoothhessian}.\vspace{1mm}
    \end{itemize}
\end{itemize}

\newpage
\section{Learning a distance function on a finite space}\label{app: finite}

\begin{algorithm}[H]
\caption{Learning a general distance function on a finite space via triplet queries}
\label{alg: finite}
\LinesNumbered
\SetKwInOut{Given}{Given}
\SetKwInOut{Output}{Output}

\Given{Input space $\mathcal{X}_o \subset \cX$\;}
\Output{Learner outputs a distance function $\hat{d}$ on $\cX_o$\;} 

Initialize $\mathcal{P} \leftarrow \emptyset$\;
Initialize $\mathcal{\cO \leftarrow \emptyset}$\;
\While{$\mathcal{P} \neq \mathcal{X}_o$}{
     Set center $x \in \mathcal{X}_o \setminus \mathcal{P}$\;
     Update $\mathcal{P} \leftarrow \mathcal{P} \cup \{x\}$\;
     Learner queries triplets to perform comparison-based merge sort to learn  latent ranking $\sigma_{x} : \mathcal{X}_o \setminus \{x\} \to [|\cX_0| - 1]$ induced by ordered set $\curlybracket{d(x, y): y \in \mathcal{X}_o \setminus \{x\}}$\;
     \tcp{\textcolor{bblue}{for all $y,y' \in \cX_o\setminus \curlybracket{x_o}$, $\sgn{d(x,y) - d(x,y')} = \sgn{\sigma_x(y) - \sigma_x(y')}$.}}
     Append $\cO \leftarrow \cO \cup \curlybracket{\sigma_x}$\;
}
Assign $\hat{d}$ consistent with $\cO$\;
\tcp{\textcolor{bblue}{for each $x \in \cX_o$, $\hat{d}(x,y) := \sigma_x(y)$.}}
Output $\hat{d}$\;
\end{algorithm}

In this appendix, we consider distance functions defined over finite sets of objects. In Section~\ref{sec: add}, we stated a bound on the query complexity of learning a distance function on finite samples. Here, we first provide a procedure for this in \algoref{alg: finite}, followed by the proof of \thmref{thm: finite}.

As stated in \thmref{thm: finite}, we assume that \(\mathcal{X}_o \subset \cX\) is a finite set of objects or samples. A \tt{finite sample distance function} \(d': \mathcal{X}_o \times \mathcal{X}_o \to \mathbb{R}_{\geq 0}\) assigns a non-negative real number to each pair of objects in \(\mathcal{X}_o\), capturing the notion of dissimilarity or distance between them. Since $d'$ is defined on a discrete set of objects it is rather straight-forward to query learn it. Essentially, \thmref{thm: finite} states that on the finite sample space $\cX_o$, using triplet queries dependent quadratically on the size of the space a learner can sufficiently find a \(d'\) that is triplet equivalent to \(d\). 

Now, we restate the theorem and provide the proof.
\begingroup
\renewcommand\thetheorem{\ref{thm: finite}}  
\begin{theorem}
    Given any distance function \(d: \cX \times \cX \to \reals_{\ge 0}\) on a sample space \(\cX\), for any finite subset \(\cX_o \subset \cX\), a learner can query to learn a finite sample distance function \(\hat{d}: \cX_o \times \cX_o \to \reals_{\ge 0}\) that is triplet equivalent to \(d\) using at most \(|\cX_o|^2 \log |\cX_o|\) triplet queries.
\end{theorem}
\endgroup
\begin{proof}
    For any fixed center \(x \in \cX_o\), the learner determines the ordering of the remaining points by increasing
distance from $x$:
    \begin{align*}
         d(x,y_1) \leq d(x,y_2) \leq \ldots \leq d(x,y_{|\cX_o|-1}), \text{ where } \curlybracket{y_i}_{i \in [|\cX_o| - 1]} = \cX_o \setminus \{x\}
    \end{align*}
    over \(\cX_o \setminus \{x\}\), yielding a permutation \(\sigma_x: \cX_o\setminus\{x\} \to [|\cX_0| - 1]\) such that for all $y,y' \in \cX_o\setminus \curlybracket{x_o}$ $\sgn{d(x,y) - d(x,y')} = \sgn{\sigma_x(y) - \sigma_x(y')}$.

    Using comparison-based merge sorting, the learner can find $\sigma_x$ with a query complexity of \((|\cX_o|-1)\log (|\cX_o|-1)\). Since there are \(|\cX_o|\) elements and the ranking for each center is independent of the others, the learner can find \(d\) restricted to \(\cX_o\) up to triplet equivalence with at most \(|\cX_o|^2\log |\cX_o|\) triplet queries.
    
\end{proof}

\section{Learning a smooth distance function via $\ell_2$ covering of the space}\label{app: l2}

In this appendix, we discuss the additive approximation of a distance function as shown in \algoref{alg: naive_pseudo2} and provide the proofs of the key results in \secref{sec: add}: \lemref{lem: taylorlemma} (in \appref{subapp: taylor}), \thmref{thm: smoothl2} (in \appref{subapp: theoreml2}), and \corref{cor: smoothl2} (in \appref{subapp: corollaryl2}).

First, we discuss the Taylor's approximation of a distance function that satisfies some smoothness property in \secref{subapp: taylor}.

\subsection{Taylor's approximation and Proof of \lemref{lem: taylorlemma}}\label{subapp: taylor}

For a function $f: \reals^p \to \real$ that is $C^{3}$, Taylor's theorem suggests local linear behavior around a sample. The formal statement is given in this theorem.

\begin{theorem}(\cite[Theorem 2.6.8]{folland2002advanced})\label{thm:taylorseries}
    Suppose \( f : \mathbb{R}^p \rightarrow \mathbb{R} \) is of class \( C^{3} \) on an open convex set \( S \).
If \( a \in S \) and \( a + \mathbf{h} \in S \), then
\begin{equation}
f(a + \mathbf{h}) =f(a) + \partial f(a) \cdot \mathbf{h} + \frac{1}{2} \mathbf{h}^{\top} (\sf{H}_a f) \mathbf{h} + R_{a,2}(\mathbf{h}),
\end{equation}
where $R_{a,2}: \reals^p \to \reals$ is a remainder term and $\sf{H}_a f$ is the Hessian matrix of $f$ at $a$.
If for all $(i_1,i_2,i_3) \in [p]^3$, the third partial derivative is bounded, i.e, $\left\lvert\frac{\partial^3 f(a)}{\partial x_{i_1}\partial x_{i_2}\partial x_{i_3}} \right \rvert \le M$ for some constant $M > 0$, then
\begin{align}
    |R_{a,2}(\mathbf{h})| \le \frac{M p^{\frac{3}{2}}}{3!} \cdot||\mathbf{h}||^{3}_2.  \label{eq: exp2}
\end{align}
\end{theorem}

Now, consider a distance function $d$ defined on a sample space $\cX$ which is $C^3$. We are interested in using \thmref{thm:taylorseries} to devise local linear approximation to $d$.

Note that, we can write
\begin{align}
    d(x,x') = d(x,x) + (x'-x)^{\top}\cdot{\partial}_{x} d(x,\cdot) + \frac{1}{2}(x'-x)^{\top} \sf{H}_x d(x,\cdot) (x'-x) + R_{x,2}(x'-x) \label{eq: exp1}
\end{align}
where $\sf{H}_x d(x,\cdot)$ is the Hessian of the function $d(x,\cdot)$ at $x$.


However, since $d$ is a distance function and also a $C^3$ map, we note that
for all $x \in \cX$ $d(x,x) = 0$ and $\partial_x d(x,\cdot) = 0$ as $d(x, y)$ achieves a local minimum at $y = x$.

Now, \eqnref{eq: exp1} and \eqnref{eq: exp2} gives the following bound
\begin{subequations}\label{eq: 1}
\begin{align}
    d(x,x') &= \frac{1}{2}(x'-x)^{\top} \sf{H}_x d(x,\cdot) (x'-x) + \Tilde{O}(||x'-x||_2^{3})\\
    \implies d(x,x') - \xi& \le  \frac{1}{2}(x'-x)^{\top} \sf{H}_x d(x,\cdot) (x'-x) \le d(x,x') + \xi, \text{ where } \xi = \frac{M p^{\frac{3}{2}}}{6} \cdot||x - x'||^{3}_2
\end{align}
\end{subequations}
But this implies that for all $x ,x' \in \cX$
\begin{equation*}
    \left|d(x,x') - \frac{1}{2}(x'-x)^{\top} \sf{H}_x d(x,\cdot) (x'-x) \right| \le \frac{M p^{\frac{3}{2}}}{6} \cdot||x - x'||^{3}_2
\end{equation*}
Thus, if $d$ is a distance function with bounded third partial derivative with a constant $M > 0$, we achieve the claim in \lemref{lem: taylorlemma}.

So, for all distances centered at $x$ the inner product $(x'-x)^{\top} \sf{H}_x d(x,\cdot) (x'-x)$ is a tight linear approximation to $d$ with error at most $\xi^3$ within a ball of $\ell_2$ norm up to $\xi$. 

Note that $d(x,\cdot)$ is a locally convex function at $x$ since it is $C^3$ and achieves a local minimum at $x$, thus the Hessian operator $\sf{H}_x$ is positive semi-definite. Since it is symmetric, we obtain a local Mahalanobis distance at $x$ in a $\xi$ neighborhood around $x$. 

\subsection{Proof of \thmref{thm: smoothl2}}\label{subapp: theoreml2}
    The proof requires a global approximation using an $\epsilon$-cover wrt $\ell_2$ distance.
    
   Consider an $\epsilon$-cover $\cC(\cX,\epsilon,\ell_2)$, denoted as $\cC(\cX)$ (in short) with centers $\cC$. The learner queries triplet comparisons to learn a finite-sample distance function $(\cC, d)$ using the strategy in \algoref{alg: naive_pseudo2}, with the underlying distance function denoted as $\hat{d}$ such that
        \begin{align*}
            x,y,z \in \cC,\quad (x,y,z)_{\hat{d}} \textit{ iff } (x,y,z)_{d}
        \end{align*}
        For query learning a finite sample distance function,  $|\cC(\cX,\epsilon,\ell_2)|^2 \log |\cC(\cX,\epsilon,\ell_2)|$  many triplet queries are sufficient to learn $\hat{d}$ as shown in \thmref{thm: finite}.

    Now, we discuss how a learner uses $\hat{d}$ to find the approximation $d'$ for $d$, i.e., assigns comparisons for $d'$ for a given triplet $(x_1,x_2,x_3) \in \cX^3$. 
    
    Denote by $\sf{NN}_2: \cX \to \cX$, a nearest neighbor map based on $\ell_2$ distances of any sample $x \in \cX$ (from the centers $\cC$). Thus, 
    \begin{align*}
        \sf{NN}_2(x) = \argmin_{x' \in \cC} \sqrt{(x-x')^{\top}(x-x')}
    \end{align*}
    We break the ties arbitrarily. This could potentially create some error in assigning the right comparison, which would at the worst only add a constant multiple of $\epsilon$ on the error. Denote by $\hat{x}$, the nearest neighbor of $x \in \cX$, i.e.  $\hat{x} \in \sf{NN}_2(x)$.
    $d'$ assigns the comparison on the triplet $\hat{x}_1, \hat{x}_2,\hat{x}_3$ as follows:
    \begin{align}\label{eq: approxmetric}
        (x_1,x_2,x_3)_{d'} = \begin{cases}
            (\hat{x}_1,\hat{x}_2,\hat{x}_3)_{\hat{d}} & \textit{ if } 
            \hat{x}_1 \neq \hat{x}_2 \textit{ or } \hat{x}_1 \neq \hat{x}_3\\
            d'(x_1,x_2) = d'(x_1,x_3)  & \textit{ o.w. }   
        \end{cases}
    \end{align}

     So, the learner answers: if ``$x_2$ is closer to $x_1$ than $x_3$'' by first finding the nearest neighbors of these points and then checking the corresponding relation to these in the distance function $(\cC, \hat{d})$. 
     Thus, the distance function $d'$ is defined as follows: for all $x,y \in \cX$
     \begin{equation}
         d'(x,y) := \hat{d}(\hat{x}, \hat{y}) \label{eq: distancedef}
     \end{equation}
    Now, we will show the approximation guarantee for $d'$. 
    
    Note, that the specific choice of the cover $\cC(\cX)$, and local approximations in \eqnref{eq: 1} and \lemref{lem: taylorlemma}
    ensures that if any points $y,z$ are far apart wrt a given $x$ then the distances are correctly captured, i.e. for any $x,y \in \cX$:
    \begin{align}
        d(x,y) - 2L\cdot\epsilon^{\alpha} \le d(\hat{x}, \hat{y}) \le  d(x,y) + 2L\cdot\epsilon^{\alpha} \label{eq: apdist1}
    \end{align}
    We can show this as follows. Using triangle inequality (see \assref{assump: a2}) we have
    \begin{align*}
        d(x,y) \le d(x,\hat{x}) + d(\hat{x},y)\\
        d(\hat{x},y) \le d(\hat{x},\hat{y}) + d(\hat{y},y)
    \end{align*}
    Summing these
    \begin{align}
        d(x,y) \le d(\hat{x},\hat{y}) + d(x,\hat{x}) + d(\hat{y},y) 
        \le d(\hat{x},\hat{y}) + 2L\cdot\epsilon^{\alpha}
    \end{align}
    where $d(\hat{y},y) \le L\cdot\epsilon^{\alpha}$ follows from cover argument and $d(x, \hat{x}) \le L\cdot\epsilon^{\alpha}$ is due to the smoothness property.
    Similarly, we get the lower bound
    \begin{align*}
       d(\hat{x},\hat{y}) - 2L\cdot\epsilon^{\alpha} \le d(x,y)
    \end{align*}
    Thus, we have 
    \begin{align*}
        d(x,y) - 2L\cdot \epsilon^{\alpha} \le d(\hat{x}, \hat{y}) \le  d(x,y) + 2L\cdot \epsilon^{\alpha}
    \end{align*}

    But since $\epsilon \le \paren{\frac{\omega}{4L}}^{\frac{1}{\alpha }}$ we have for all  $x,y,z \in \cX$, if  $d(x,y) > d(x,z) + \omega$ then
     \begin{align*}
    d(\hat{x},\hat{y}) \ge d(x,y) - 2L\cdot \epsilon^{\alpha} > d(x,z) + \omega - 2L\cdot \epsilon^{\alpha}
        \ge d(\hat{x},\hat{z}) + \omega - 4L\cdot \epsilon^{\alpha} \ge d(\hat{x},\hat{z})
    \end{align*}
    Thus, we achieve the stated claim of the theorem.

\subsection{Proof of \corref{cor: smoothl2}}\label{subapp: corollaryl2}

The proof follows directly from \thmref{thm: smoothl2}. First, observe that since \( d \) satisfies \assref{assump: a1} with constant \( M > 0 \) and the eigenvalues of the Hessian matrix \(\sHx\) are uniformly bounded above by \(\gamma > 0\) for all \( x \in \cX \), it implies that \( d \) is \((\alpha, L, \delta)\)-smooth with \(\alpha = 2\), \(\delta = 1\), and \( L = \frac{\gamma}{2} + \frac{Mp^{3/2}}{6} \). Additionally, since \( d \) satisfies the triangle inequality of \assref{assump: a2}, for \( \epsilon < \left( \frac{\omega'}{2\gamma + 4K} \right)^{1/2} \), where \( \omega' = \min\{1, \omega\} \), the result of the corollary is established.

\newpage
\section{Learning a Mahalanobis distance function}\label{app: maha}
In this Appendix, we show the proof of \thmref{thm: mahamain} and \thmref{thm: advmahamain}. First, we show a construction of a basis of the space $\symm$  in \lemref{lem: orthoset} and then show it is is well-conditioned in \lemref{lem: condition}. 

We provide an extended version of \algoref{alg:learnapproxmaha} (see \secref{sec: mult}) for two settings in \figref{fig:extendedalgos}: \algoref{alg: learnapproxmahamain1} for the noiseless learning of Mahalanobis distance function (see \thmref{thm: mahamain}) and \algoref{alg: learnapproxmahamain} for the noisy setting where labels are provided according to latent distance function $d$ (see \thmref{thm: advmahamain}). For the proof in this appendix, we use the following notation:
\paragraph{Notations:} For a matrix \( M \in \reals^{p \times p} \) and indices \( i,j \in \{1, \dots, p\} \), \( M_{ij} \) denotes the entry of \( M \) at the \(i\)-th row and \(j\)-th column. Matrices are denoted by \( M, M', N \). Unless stated otherwise, the target matrix for learning a Mahalanobis distance fuction is denoted as \( M^* \). The null space of a matrix \( M \), i.e., \( \{x \in \reals^d \mid Mx = 0\} \), is denoted by \( \nul{M} \).

The eigenvalues of a matrix are denoted by \( \gamma, \lambda, \gamma_i, \lambda_i \), and the corresponding eigenvectors by \( \mu_i, u_i, v_i \) (where the eigenvectors are orthogonal). The element-wise inner product of two matrices \( M, M' \) is defined as \( \inner{M', M} := \sum_{i,j} M'_{ij} M_{ij} \). Vectors in \( \reals^p \) are denoted by \( x, y, z \). Note that for any \( x \in \reals^p \), \( \inner{xx^{\top}, M} = x^{\top} M x \). For simplicity, we also write the inner product as \( M \idot M' \).

For analysis, we express \( p \times p \) symmetric matrices as vectors of size \( p(p+1)/2 \). For any symmetric matrix \( A \in \reals^{p \times p} \), let \( \text{vec}(A) \in \reals^{p(p+1)/2} \) have \( p \) entries indexed by \( i \in \{1, \dots, p\} \) and \( {p \choose 2} \) entries indexed by \( (i, j) \), for \( 1 \leq i < j \leq p \), where:
\begin{align*}
    \text{vec}(A)_i &= A_{ii}, \\
    \text{vec}(A)_{ij} &= \frac{A_{ij} + A_{ji}}{\sqrt{2}}.
\end{align*}

For symmetric matrices \( A \) and \( B \), we have \( A \idot B = \text{tr}(AB) = \text{vec}(A) \cdot \text{vec}(B) \), where the first term is a matrix dot product, and the second is a vector dot product. Hence, \( \|\text{vec}(A)\|^2 = \|A\|_F^2 \) (Frobenius norm).


\subsection{Construction of a basis with bounded condition number}\label{subapp: condition}
 Although one can construct a basis for the space of $p \times p$ symmetric matrices in different ways, in \lemref{lem: orthoset} the constructed set is helpful in the sense that the vectorization of the basis gives a well-conditioned operator as shown in \lemref{lem: condition}.

Let $\ue$ consist of $p(p+1)/2$ unit vectors in $\reals^p$, as follows:
$$ \ue = \{e_i: 1 \leq i \leq p\} \cup \{(e_i + e_j)/\sqrt{2}: 1 \leq i < j \leq p \} .$$
Here $e_i$ is the $i$th standard basis vector.

Let $\M$ be the $(p+1)/2 \times (p+1)/2$ design matrix whose $i$th row is $\mbox{vec}(u_i u_i^T)$ for $u_i \in \ue$ arranged in the following way.

\begin{gather}\label{eq: design}
        \mathbb{D} = \begin{pmatrix}
\text{vec}(e_1 e_1^\top)^\top \\
\vdots\\
\text{vec}(e_pe_p^{\top})^\top\\
\frac{1}{2}\text{vec}((e_1 + e_2)(e_1 + e_2)^{\top})^\top \\
\vdots \\
\frac{1}{2}\text{vec}((e_{p-1} + e_p) (e_{p-1} + e_p)^\top)^\top 
\end{pmatrix}
\end{gather}

Our first result on $\M$ is to show that it is full rank, i.e., $\rank{\M} = p(p+1)/2$ in \lemref{lem: orthoset}.

\begin{lemma}\label{lem: orthoset} Consider the design matrix $\M$ as constructed in \eqnref{eq: design}. The rank of $\M$ is $p(p+1)/2$ in the vector space $\symm$.

\end{lemma}

\begin{proof}
Consider a symmetric matrix $X \in \symm$. We wish to solve 
\begin{align*}
    \M vec(X) =  0
\end{align*}
But this gives us the following system of linear equations:
\begin{subequations}\label{eq: sys}
\begin{gather}
    \forall i \in \bracket{p},\quad X_{ii} = 0 \label{eq: sys1}\\
       \forall (i < j),\quad  (X_{ii} + 2X_{ij} + X_{jj}) = 0 \label{eq: sys2}
\end{gather}
\end{subequations}

It is straightforward that we can achieve the equality constraints in \eqnref{eq: sys} only if X = \textbf{0} (a zero matrix). This implies $\M$ is full rank in the space of symmetric matrices, and hence the claim of the lemma is proven.    
\end{proof}

To show that the matrix $\M$ we ultimately obtain is not too strongly affected by small errors in estimating the $c_i^*$ (see \algoref{alg:learnapproxmaha} or \algoref{alg: learnapproxmahamain1}), we need to bound the operator norm of $\mathbb{D}^{-1}$. We show that it is at most $\max(2p-1,4)$, which is $\leq 2p$ for $p \geq 2$.
\begin{lemma}\label{lem: condition}
For any symmetric $p \times p$ matrix $X$, we have 
$$\| \M \mbox{\rm vec}(X) \|^2 \geq \frac{1}{\max(2p-1,4)} \|X\|_F^2 .$$
\end{lemma}

\begin{proof}
Using $i$ and $(i,j)$ indexing for vectors in $\reals^{p(p+1)/2}$ introduced above, we get
\begin{align*}
( \M \mbox{vec}(X))_i &= \mbox{vec}(e_ie_i^T) \cdot \mbox{vec}(X) = X_{ii} \\
(\M \mbox{vec}(X))_{ij} &= \mbox{vec}((e_i+e_j)(e_i+e_j)^T/2) \cdot \mbox{vec}(X) = \frac{1}{2} (X_{ii} + X_{jj} + X_{ij} + X_{ji} )
\end{align*}
Let $R_{ij} = X_{ii} + X_{jj} + X_{ij} + X_{ji}$. Then
$$ \|\M \mbox{vec}(X) \|^2 = \sum_{i=1}^p X_{ii}^2 + \frac{1}{4}\sum_{i<j} R_{ij}^2 $$
and
\begin{align*}
\|X\|_F^2 
&= \sum_i X_{ii}^2 + 2 \sum_{i < j} X_{ij}^2 \\
&= \sum_i X_{ii}^2 + \frac{1}{2} \sum_{i < j} (R_{ij} - X_{ii} - X_{jj})^2 \\
&\leq \sum_i X_{ii}^2 + \frac{1}{2} \sum_{i<j} (2R_{ij}^2 + 2(X_{ii} + X_{jj})^2) \\
&\leq \sum_i X_{ii}^2 + \frac{1}{2} \sum_{i<j} (2R_{ij}^2 + 4X_{ii}^2 + 4X_{jj}^2) \\
&= (2p-1) \sum_i X_{ii}^2 + \sum_{i<j} R_{ij}^2 
\end{align*}
where we have twice used the inequality $(a+b)^2 \leq 2a^2 + 2b^2$. Comparing the two expressions, we see that $\|\M \mbox{vec}(X) \|^2 \geq \|X\|_F^2 / (\max(2p-1,4))$.
\end{proof} 

\subsection{Proof of \thmref{thm: mahamain}}\label{subapp: mahamain}
    Consider the set of the standard basis of $\reals^p$, denoted as $U = \curlybracket{e_i : 1 \le i \le p}$. First, the learner queries the user to find a unit vector $y \in \reals^d$ such that 
    \begin{align}
        y = \argmax_{u \in U} u^{\top}M^*u \label{eq: ymax}    
    \end{align}
     This can be achieved in at most $p$ queries by iteratively querying triplets of the form \(\{(0, u_i, u_j) : u_i \neq u_j \in U\}\). 
     Note that this corresponds to the largest diagonal entry of $M^*$. 
     
     In this proof, we assume that $y^{\top}M^*y = 1$. The learner queries to find a symmetric matrix $M'$ such that $\|M^* - M'\|_{F}$ is small and $y^\top M^* y = y^\top M'y = 1$.

    
     Let $U_{\sf{ext}}$ consist of $p(p+1)/2$ unit vectors in $\reals^p$, as follows:
     \begin{align}
        U_{\sf{ext}} = \{e_i: 1 \leq i \leq p\} \cup \{(e_i + e_j)/\sqrt{2}: 1 \leq i < j \leq p \}. \label{eq: basis}
     \end{align}
Here $e_i$ is the $i$th standard basis vector.
     
     Now, note that there exists scalars $\curlybracket{c_i^* \ge 0}_{i \in \bracket{\frac{p(p+1)}{2}}}$ such that $M^*$ satisfies the following set of equations 
     \begin{equation}
        \forall u_i \in U_{\sf{ext}}, \quad u_i^\top M^* u_i = c_i^* y^T M^* y\label{eq: setofeq}
     \end{equation}
    The key idea to learn $M^*$ is to approximate the scalars $\curlybracket{c_i^*}_{i \in \bracket{ \frac{p(p+1)}{2}}}$ within an $\epsilon$ bound which can be used to show the guarantee in the statement of the theorem. Using a useful observation, we can bound the scalars in terms of the condition number of $M^*$ 
    \begin{align}
       \forall u_i \in \ue,\quad \kappa(M^*) \ge \frac{\lambda_{\sf{max}}(M^*)}{\lambda_{\sf{min}}(M^*)} \ge c_i^* = \frac{u_i^\top M^*u_i}{y^\top M^*y} \ge \frac{\lambda_{\sf{min}}(M^*)}{\lambda_{\sf{max}}(M^*)} = \frac{1}{\kappa(M^*)} \label{eq: condition}
    \end{align}
    If $M^* \not \succ 0$ then we consider the definition where $\lambda_{\sf{min}}$ is the smallest non-zero eigenvalue of $M^*$. In the remainder of the proof below, we follow this convention in which the indexing of $u = u_i \in U_{\sf{ext}}$ can be used alternately, i.e., $\_i$ or $\_u$ when the context is clear.


    As shown in \algoref{alg: learnapproxmahamain1}, for a given rank-1 matrix $uu^\top$ where $u \in U_{\sf{ext}}$, learner can set $\hat{c}_u = 1$ and query $(0, \sqrt{\hat{c}}_u y,u)$. Now, the learner can update\footnote{Note: learner picks $M$ such that $y^\top M y = 1$}  $\hat{c}_u$ via binary search as follows:
    \begin{align}
        \hat{c}_u = \begin{cases}\label{eq: cases}
            H \gets \hat{c}_u, \hat{c}_u \gets \frac{H + L}{2} & \textnormal{ if } (0, \sqrt{\hat{c}_u} y, u)_\ell = ``\sf{>}" \implies u^\top M^* u < \hat{c}_u\\ 
            L \gets \hat{c}_u, \hat{c}_u \gets \frac{H + L}{2} & \textnormal{ if } (0, \sqrt{\hat{c}_u}y, u)_\ell = ``\sf{<}" \implies u^\top M^* u > \hat{c}_u\\
            \text{return } \hat{c}_u & \textnormal{  o.w.}
        \end{cases}
    \end{align}
    where scalar estimates $H$\footnote{In \algoref{alg: learnapproxmahamain1}, $H$ is set to $\infty$ at start, so exact update in second line is $L \gets \hat{c}_u, \hat{c}_u \gets \min\{ 2\hat{c}_u , \frac{H + L}{2}\}$. } and $L$ are updated iteratively to keep an estimate of the target scalar $c_u^*$. Initially, $H$ is set to $\infty$ (see \algoref{alg: learnapproxmahamain1}). Thus, binary search first estimates how large is $c^*_u$ (which can be as large as $\kappa(M^*)$) in $\log \kappa(M^*)$ iterations with the doubling method, and then $H$ is updated.
    Now, note that binary search achieves a closeness of $\epsilon$ on interval size $\kappa(M^*)$, i.e., $|H - L| \le \epsilon$ (see \algoref{alg: learnapproxmahamain1}) after $\paren{\log{\frac{\kappa(M^*)}{\epsilon}}}$ updates. 
    This implies that $\hat{c}_u \in \bracket{c_u^* - \epsilon, c_u^* + \epsilon}$, where $H$ and $L$ are values within $\epsilon$-neighborhood of $c^*_u$. Since $u \in U_{\sf{ext}}$ was arbitrarily chosen we have for all vectors $u \in U_{\sf{ext}}$
    \begin{gather}
        |c_u^* - \hat{c}_u| \le \epsilon \label{eq: c}
    \end{gather}
    Without loss of generality we assume that for all $u \in \ue$, $\hat{c}_u \ge c_u^*$ which can be achieved by setting $\hat{c}_u$ to the final update on the counter $H$ (see \algoref{alg: learnapproxmahamain1}) after all the updates.

    Using the set of learnt scalars $\curlybracket{\hat{c}_u}_{u \in \ue}$, learner picks a $M' \in \symm$ such that for the fixed choice of $y$ in \eqnref{eq: ymax}, $y^\top M' y = 1$ and
    $$\forall u \in \ue, \quad u^\top M' u = \hat{c}_u,\quad $$
    Now, we will show that Frobenius norm of the difference $M^* - M'$ is bounded by a function of $\epsilon$ which completes the proof of the theorem.

    Consider the following construction based on vectorization of rank-1 matrices $\{uu^\top\}_{u \in \ue}$ and $M^* - M'$
    \begin{subequations}\label{eq: mat1} 
    \begin{gather*}
        \mathbb{D} \text{vec}(M^* - M') = C, \textnormal{where}\\
\,\, \mathbb{D}[i] = \text{vec}(e_ie_i^\top),\, \forall\, 1 \le i \le p,\quad \mathbb{D}[(i,j)] = \text{vec}\paren{\frac{1}{2}(e_i + e_j)(e_i + e_j)^{\top}},\,\, \forall\, 1 \le i < j \le p,\\
{C}^\top = (
c_1^* - \hat{c}_1,\,\,
c_2^* - \hat{c}_2,\,\, 
\ldots\,\, 
c_{p(p+1)/2}^* - \hat{c}_{p(p+1)/2})^\top
    \end{gather*}
\end{subequations}
    Now, using \lemref{lem: condition} we can bound the Frobenius norm of $M^* - M'$ as 
    \begin{align*}
        ||C||_2^2 = ||\M \text{vec}(M^* - M')||^2_2 \ge \frac{1}{\max(2p-1,4)} ||M^* - M'||_F^2
    \end{align*}

    Now, rewriting the norm $||C||_2$, we have 
    \begin{align}
         || M^* - M'||_F &\le \sqrt{2p}\cdot ||C||_2 \nonumber\\
         &\le  \sqrt{2p}\cdot \sqrt{\sum_i (c_i^* - \hat{c}_i)^2} \label{eq: b33}\\
         &\le \sqrt{2p}\cdot \sqrt{\sum_i \epsilon^2} \label{eq: b4}\\
        &\le \sqrt{2p}\cdot \sqrt{\frac{p(p+1)}{2} \epsilon^2} \label{eq: b5}\\
         & = \sqrt{p} \cdot \sqrt{p(p+1)} \epsilon\nonumber\\
         & \le p^2 \epsilon\nonumber
\end{align}

     In \eqnref{eq: b33}, we rewrite $||C||_2 = \sqrt{\sum_i \paren{{c_i^*} -{\hat{c}_i}}^2}$ by definition. In \eqnref{eq: b4}, we use the bounds from \eqnref{eq: c}. In \eqnref{eq: b5}, we use the dimension of $C$ which is bounded by $p^2$.

    Now, to achieve the $\epsilon$ closeness to $c^*_i$ for all $i \in [p(p+1)/2]$ the learner makes $\paren{\log \frac{\kappa(M^*)}{\epsilon}} + \log \kappa(M^*)$ triplet queries. Since there are $\frac{p(p+1)}{2}$ elements in $\ue$ the learner needs to query at least $\frac{p(p+1)}{2} \log \frac{\kappa(M^*)^2}{\epsilon}$ times during the binary search. Beyond this, learner queries $p$ times to find $y$ (see \eqnref{eq: ymax}). 
    Hence, in total, the learner makes at most
    \begin{align*}
        \frac{p(p+1)}{2} \log \paren{\frac{\kappa(M^*)^2}{\epsilon}} + p 
    \end{align*}
    triplet queries to achieve a bound of $p^2 \epsilon$ on the Frobenius norm of $M^* - M_0$. Alternately, this implies that the learner achieves $\epsilon$ error with $\frac{p(p+1)}{2} \log \paren{\frac{p^2 \kappa(M^*)^2}{\epsilon}} + p$ triplet queries. 
    
    Note that $M' \in \symm$. 
    To achieve a positive semi-definite matrix, the learner projects\footnote{This can be achieved by eliminating all the eigenvectors with negative eigenvalues.} $M'$ onto the space of positive semi-definite matrices $\symmp$; using which it obtains $M_{\sf{psd}} \in \symmp$. Noting that $M_{\sf{psd}}$ minimizes $||M' - M_0||$ for $M_0 \succeq 0$ (see Page 399,~\citet{Boyd_Vandenberghe_2004}), we have 
    \begin{align*}
        ||M^* - M_{\sf{psd}}||_F \le ||M^* - M'||_F + ||M' - M_{\sf{psd}}||_F \le 2\epsilon
    \end{align*}
    Thus, we have shown that the learner can triplet query to find $M_{\sf{psd}} \in \symmp$ such that 
    $$||M^* - M_{\sf{psd}}||_F \le \epsilon$$
    with $\frac{p(p+1)}{2} \log \paren{\frac{2p^2 \kappa(M^*)^2}{\epsilon}} + p$ triplet queries.
    
    Hence, this completes the proof.
\newpage

\subsection{Proof of \thmref{thm: advmahamain}}\label{subapp: advmahamain}
 First, note that upon any query of the form \( (x, x', x'') \in \cX^3 \), the user provides labels from the set \( \{-1, 0, 1\} \) with respect to the underlying distance function \( d \), not \( d_{\sHx} \) (or \( \sHx \)). We provide the modified version of the \algoref{alg:learnapproxmaha} in this setting in \algoref{alg: learnapproxmahamain}.

Despite this, we show that the learner can still approximate the Mahalanobis distance function \( d_{\sHx} \) corresponding to the target Hessian matrix \( \sHx \). The key to this approximation lies in minimizing the noise in the labels, which can be reduced to a small error, as demonstrated in \lemref{lem: taylorlemma}.

Consider a sample \( x' \in \cX \) within a neighborhood of \( x \) defined by some scalar \( \rho > 0 \), i.e., \( \|x - x'\|_2 \leq \rho \). By \lemref{lem: taylorlemma}, we have
\[
    \left \lvert d(x, x') - \frac{1}{2}(x - x')^{\top} \sHx (x - x') \right \rvert \leq \frac{Mp^{1.5}}{6} \rho^3.
\]
Following the construction in \eqnref{eq: ymax}, let \( y \) be chosen from \( B_2(x, \rho) \) such that (since \( \sHx \) is full rank, such a \( y \) exists):
\begin{align}
        y = \argmax_{u \in U} (x - \rho u)^{\top} \sHx (x - \rho u). \label{eq: ymaxx}    
\end{align}
Define \( \tau_x := \frac{1}{y^{\top} \sHx y} \). The learner may now model the problem of learning \( \sMx := \tau_x \sHx \), similar to \eqnref{eq: setofeq} in the proof of \thmref{thm: mahamain}. That is, there exist constants \( \{c_i^* > 0\}_{i \in \bracket{p(p+1)/2}} \) such that
\begin{align}
   \forall u_i \in \ue,\quad u_i^{\top} \sMx u_i = c_i^* (y^T \sMx y) \label{eq: alleq} 
\end{align}
for a set of vectors \( \ue \) as defined in \eqnref{eq: basis}. In this case, the learner queries triplets of the form \( (x, x + \rho \sqrt{c_i} y, x + \rho u_i) \) for \( u_i \in \ue \) and estimates \( c_i^* \), allowing it to find an approximation \( \sf{M}'_x \) with bounded Frobenius norm relative to \( \sMx \). Note that $\frac{1}{\gm} \ge \tau_x \ge \frac{1}{\gM}$.

Following the strategy of querying triplets of the form \( (x, x + \rho \sqrt{c_i} y, x + \rho u_i) \) (see \algoref{alg: learnapproxmahamain}), the learner receives a label from the user corresponding to
\begin{align}
    \sgn{d(x, x + \rho \sqrt{c_i} y) - d(x, x + \rho u_i)}. \label{eq: sign}
\end{align}
Expanding \eqnref{eq: sign}, we have:
\[
    \sgn{\rho^2 u_i^{\top} \sHx u_i + R_{x,2}(\rho u_i) - \rho^2 c_i y^{\top} \sHx y - R_{x,2}(\rho \sqrt{c_i} y)},
\]
where \( R_{x,2}(\rho u_i) \) and \( R_{x,2}(\rho \sqrt{c_i} y) \) represent remainder terms in the Taylor approximation, as shown in \thmref{thm:taylorseries}.
After some simplification, this becomes:
\[
    \rho^2 u_i^{\top} \sHx u_i - \rho^2 \left( c_i y^{\top} \sHx y - \frac{R_{x,2}(\rho u_i)}{\rho^2} + \frac{R_{x,2}(\rho \sqrt{c_i} y)}{\rho^2} \right).
\]

Since \( u_i \) and \( y \) are unit vectors, we have:
\[
    \left| \frac{R_{x,2}(\rho u_i)}{\rho^2} - \frac{R_{x,2}(\rho \sqrt{c_i} y)}{\rho^2} \right| \leq \frac{Mp^{1.5}}{6} \left( 1 + c_i^{1.5} \right) \rho.
\]

Now, for a given \( \rho \), define the scalar functions \( f_{\rho}: \reals \to \reals \) and \( \eta_{\rho}: \reals \to \reals \) as:
\[
    f_{\rho}(\mu) := \frac{R_{x,2}(\rho u_i)}{\rho^2} - \frac{R_{x,2}(\rho \sqrt{\mu} y)}{\rho^2}, \quad \eta_{\rho}(\mu) := \frac{Mp^{1.5}}{6} \left( 1 + \mu^{1.5} \right) \rho.
\]
First, note that with the definition of \( f_{\rho} \), we can rewrite equation \eqref{eq: sign}. On the triplet query \( (x, x + \rho \sqrt{c_i} y, x + \rho u_i) \), the user provides the label corresponding to
\begin{align}
    \sgn{u_i^{\top} \sHx u_i - c_i \cdot y^{\top} \sHx y + f_{\rho}(c_i)} \label{eq: mahaeq}
\end{align}

Now, assume that \algoref{alg: learnapproxmahamain} is called with an error threshold $\epsilon > 0$. Note that if $c_i > c_i^* + \epsilon$ then we have 
\begin{align}
    c_i y^{\top} \sHx y - u_i^\top \sHx u_i > (c_i^* + \epsilon) y^{\top} \sHx y - c_i^* y^{\top} \sHx y = \epsilon y^{\top} \sHx y \ge \epsilon \gm \label{eq: minbound}
\end{align}
where we have used the equality $u_i^{\top} \sMx u_i = c_i^* (y^T \sMx y)$ from \eqnref{eq: alleq}.

Now, the largest deviation $f_{\rho}(c_i)$ can incur to \eqnref{eq: mahaeq} is bounded by
\begin{align*}
    |f_{\rho}(c_i)| \le \eta_{\rho}(c_i) = \frac{Mp^{1.5}}{6} ( 1 + c_i^{1.5})\rho
\end{align*}
But if $\eta_{\rho}(c_i)$ is smaller than the lower bound $ \epsilon \gm$ as shown in \eqnref{eq: minbound} then we force the binary search (in \algoref{alg: learnapproxmahamain}) to update $c_i$ to a lower value until it is at at most $c_i^* + \epsilon$, i.e.
\begin{align}
    \eta_{\rho}(c_i) = \frac{Mp^{1.5}}{6} ( 1 + c_i^{1.5})\rho < \epsilon\gm \implies c_i y^{\top} \sHx y - u_i^\top \sHx u_i + f_{\rho}(c_i) > 0 \label{eq: control}
\end{align}

Similarly, we enforce the same condition on binary search using the bound in \eqnref{eq: control} when $c_i^* > c_i + \epsilon$ where we get the following lower bound
\begin{align*}
    u_i^\top \sHx u_i - c_i y^{\top} \sHx y > \epsilon \gm
\end{align*}

Now, note that we just have to control the error term $\eta_\rho(c_i)$ with $\epsilon\gm$ which yields a bound on $\rho$ as
\begin{align*}
    \frac{Mp^{1.5}}{6} ( 1 + c_i^{1.5})\rho < \epsilon\gm \implies \rho < \frac{6\gm\epsilon}{M p^{1.5}(1 + c_i^{1.5})} 
\end{align*}
Now, since $c_i^*$ could be set to as large as $\kappa(\sHx)$ during the binary search procedure, for $\rho < \frac{6\gm\epsilon}{4M p^{1.5} \kappa(\sHx)^2}$ \algoref{alg: learnapproxmahamain} returns $c_i$ that is at most $\epsilon$ away from $c_i^*$, i.e., $|c_i^* - c_i| \le \epsilon$.

Thus, using \thmref{thm: mahamain}, \algoref{alg: learnapproxmahamain} requires \( \log \left( \frac{2p^2 \kappa(\sHx)^2}{\epsilon} \right) \) triplet queries to achieve \( \epsilon \)-closeness to $c_i^*$. 


Since the learner solves for \( \sMx \) (as in \eqnref{eq: alleq}) assuming \( y^{\top} \sMx y = 1 \), i.e., query learns \( \tau_x \sHx \), using the proof of \thmref{thm: mahamain} in \appref{app: maha}, we have the following bound:
\[
    \|\tau_x \sHx - \sf{H}_x\|_F \leq \epsilon
\]
Since the learner is unaware of the constants we assume that $\epsilon \in \paren{0, \frac{3\gm}{2M p^{1.5} \kappa(\sHx)^2}}$ so that the factor $\rho$ can be set to $\epsilon^2$. Using $\kappa(\sHx) \le \frac{\gM}{\gm}$, we have shown that for an error threshold $\epsilon \in \paren{0, \frac{3\gm^3}{2M p^{1.5} \gM^2}}$, learner finds $\tau_x \sHx$ with a Frobenius norm bound of $\epsilon$ using at most \( \frac{p(p+1)}{2}\log \left( \frac{2p^2 \gM^2}{\gm^2 \epsilon}\right)  + p\) triplet queries of the form $(x, x + \epsilon^2\sqrt{c_u}y, x + \epsilon^2 u)$  (as shown in \algoref{alg: learnapproxmahamain}) for $u \in \ue$.

This completes the proof of the theorem.

\begin{figure}
    \centering
\begin{algorithm}[H]
\caption{Learning a Mahalanobis Distance Function in \tt{noiseless} setting}
\label{alg: learnapproxmahamain1}
\SetKwInOut{Input}{Input}
\SetKwInOut{Output}{Output}
\LinesNumbered
\Input{Error threshold $\epsilon$}
\Output{Learner outputs a distance function $\hat{d}$ on the sample space $\cX$}
Initialize: $\mathcal{T} = \emptyset$\;
\SetKwProg{Fn}{Function}{}{}
\SetKwProg{elf}{if}{}{}
\Fn{BinarySearch($y$, $u$, $\epsilon$):}{
    \text{Set} $L = 0$, $H = \infty$\;
    Set $c = 1.$\;
    Repeat until  $|H - L| > \epsilon$\;
    \Begin{
        Learner queries $(0, \sqrt{c} \cdot y, u)$ and receives the label $(0, \sqrt{c} \cdot y, u)_{\ell}$\;
        \If{$\ell((0, \sqrt{c} \cdot y, u)) == ``>"$}{ \tcp{ \textcolor{bblue}{Implies $u^\top M^*u  < c$}}
            
            $H \gets c$, $c \gets \frac{H + L}{2}$\;
        }
        \ElseIf{$\ell((0, \sqrt{c} \cdot y, u)) == ``<"$}{ \tcp{ \textcolor{bblue}{Implies $u^\top M^*u  > c$}}
                $L \gets c$, $c \gets 2c\cdot \textbf{1}_{\{H = \infty\}} + \frac{H + L}{2}\cdot \textbf{1}_{\{H < \infty\}}$\;
            }
        \Else{ \tcp{ \textcolor{bblue}{Implies $u^\top M^*u  = c_i^* = c$}}
                \textbf{return} $c$\;
            }
        
    }
    End Repeat\;
}

\Fn{ApproximateDistance():}{
    Set $U := \curlybracket{e_i : i \in \bracket{p}}$ which is standard basis of $\reals^p$\;
    Learner triplet queries to find $ u_0 \in U$ such that $u_0^\top M^* u_0 = \max_{u \in U} u^\top M^* u$\; \tcp{\textcolor{bblue}{ Queries of the form $\{(0, u_i,u_j): u_i \neq u_j \in U\} $ suffice}}
    Set $y = u_0$\;
    Set $\ue := \curlybracket{e_i : i \in \bracket{p}} \cup \curlybracket{\frac{(e_i + e_j)}{\sqrt{2}} : i,j \in \bracket{p}, i\le j}$\;
    \ForEach{$i \in \{1, 2, \ldots, \frac{p(p+1)}{2}\}$}{
        \text{Set} $u_i := \ue[i]$\; 
        \text{Set} $\hat{c}_i \gets \text{BinarySearch}(y, u_i, \epsilon)$\; \tcp{\textcolor{bblue}{This finds $\hat{c_i}$, i.e. $\epsilon$-close to $c_i^*$ defined as $u_i^\top M^*u_i = c_i^* y^\top M^* y$}}
        Update $\mathcal{T} \gets \mathcal{T} \cup \{(0, \sqrt{\hat{c}_i}y, u_i)_{0}\}$\;
        \tcp{\textcolor{bblue}{Append $(0, \sqrt{\hat{c}_i}y, u_i)$ with label 0}}
    }
    \bf{return} $\cT$\; 
}
Learner generates labeled triplet queries $\cT \leftarrow \text{ApproximateDistance}()$\;\tcp{\textcolor{bblue}{Solve for all $u_i \in \ue$ $u_i^{\top} M u_i = \hat{c_i}$ in the space of $p \times p$ symmetric matrices and then project onto the semidefinite cone $\symmp$.}}
Learner outputs $\hat{d}_{M'} \in \maha$\;
\end{algorithm}
\caption{Learning a Mahalanobis distance function with triplet queries in noiseless setting where the learner receives labels to $d_{M^*}$.}
    \label{fig:extendedalgos}
\end{figure}

\begin{figure}
\quad
\hspace{1mm}
\begin{algorithm}[H]
\caption{Learning a Mahalanobis Distance Function in \tt{noisy} setting}
\label{alg: learnapproxmahamain}
\SetKwInOut{Input}{Input}
\SetKwInOut{Output}{Output}
\LinesNumbered
\Input{Pivot $x$, Error threshold $\epsilon$}
\Output{Learner outputs a distance function $\hat{d}$ on the sample space $\cX$}
\looseness-1
Initialize: $\mathcal{T} = \emptyset$\;
\SetKwProg{Fn}{Function}{}{}
\SetKwProg{elf}{if}{}{}
\Fn{BinarySearch($y$, $u$, $\epsilon$):}{
    \text{Set} $L = 0$, $H = \infty$\;
    Set $c = 1.$\;
    Repeat until  $|H - L| > \epsilon$\;
    \Begin{
        Learner queries $(x,x+ \epsilon^2\sqrt{c} \cdot y, x+\epsilon^2 u)$ and receives the label $(x, x+\epsilon^2\sqrt{c} \cdot y, x+\epsilon^2 u)_{\ell}$\;
        \If{$\ell((x,x+ \epsilon^2\sqrt{c} \cdot y,x+ \epsilon^2 u)) == ``>"$}{ \tcp{ \textcolor{bblue}{Implies $d(x,x+ \epsilon^2\sqrt{c} \cdot y) > d(x,x+ \epsilon^2 u)$ }}
             $H \gets c$, $c \gets \frac{H + L}{2}$\;
        }
        \ElseIf{$\ell((x,x+ \epsilon^2\sqrt{c} \cdot y,x+ \epsilon^2 u)) == ``<"$}{ \tcp{ \textcolor{bblue}{Implies $d(x,x+ \epsilon^2\sqrt{c} \cdot y) < d(x,x+ \epsilon^2 u)$}}
                $L \gets c$, $c \gets 2c\cdot \textbf{1}_{\{H = \infty\}} + \frac{H + L}{2}\cdot\textbf{1}_{\{H < \infty\}}$\;
            }
        \Else{ \tcp{ \textcolor{bblue}{Implies $d(x,x+ \epsilon^2\sqrt{c} \cdot y) = d(x,x+ \epsilon^2 u)$}}
                \textbf{return} $c$\;
            }
        
    }
    End Repeat\;
}

\Fn{ApproximateDistance():}{
    Set $U := \curlybracket{e_i : i \in \bracket{p}}$ which is standard basis of $\reals^p$\;
    Learner triplet queries to find $ u_0 \in U$ such that $u_0^\top M^* u_0 = \max_{u \in U} u^\top M^* u$\; \tcp{\textcolor{bblue}{ Queries of the form $\{(0, u_i,u_j): u_i \neq u_j \in U\} $ suffice}}
    Set $y = u_0$\;
    Set $\ue := \curlybracket{e_i : i \in \bracket{p}} \cup \curlybracket{\frac{(e_i + e_j)}{\sqrt{2}} : i,j \in \bracket{p}, i\le j}$ \;
    \ForEach{$i \in \{1, 2, \ldots, \frac{p(p+1)}{2}\}$}{
        \text{Set} $u_i := \ue[i]$\; 
        \text{Set} $\hat{c}_i \gets \text{BinarySearch}(y, u_i, \epsilon)$\; \tcp{\textcolor{bblue}{This finds $\hat{c_i}$, i.e.  $\epsilon$-close to $c_i^*$ defined as $u_i^\top M^*u_i = c_i^* y^\top M^* y$}}
        Update $\mathcal{T} \gets \mathcal{T} \cup \{(0, \sqrt{\hat{c}_i}y, u_i)_{0}\}$\;
        \tcp{\textcolor{bblue}{Append $(0, \sqrt{\hat{c}_i}y, u_i)$ with label 0}}
    }
    \bf{return} $\cT$\; 
}
Learner generates labeled triplet queries $\cT \leftarrow \text{ApproximateDistance}()$\;\tcp{\textcolor{bblue}{Solve for all $u_i \in \ue$ $u_i^{\top} M u_i = \hat{c_i}$ in the space of $p \times p$ symmetric matrices and then project onto the semidefinite cone $\symmp$.}}
Learner outputs $\hat{d}_{M'} \in \maha$\;
\end{algorithm}
\caption{Learning a Mahalanobis distance function with triplet queries in noisy setting where the learner receives labels to general $d$.}
    \label{fig:extendedalgos}
\end{figure}

\newpage
\section{Learning a smooth distance function via local Mahalanobis distance functions}\label{app: hessian}
\begin{algorithm}[!]
\caption{Learning a smooth distance function via local Mahalanobis distance functions}
\label{alg: smoothdisthessianfull}
\SetKwProg{ForEach}{foreach}{}{}
\LinesNumbered
\SetKwInOut{Given}{Given}
\SetKwInOut{Output}{Output}
\Given{Approximation threshold $\omega$.}
\Output{Distance function $d'$ on $\cX$.}
Construct an $\epsilon$-cover $\mathcal{C} \subset \cX$ with respect to $\ell_2$ distance, for $\epsilon$ as given in \thmref{thm: smoothhessian}\;
Using triplet queries, construct distance function $d_f$ on $\mathcal{C}$ that is triplet-equivalent to $d$ using \algoref{alg: finite}\; \tcp{\textcolor{bblue}{$d_f$ is triplet equivalent to $d$ restricted to $\cC$ as shown in \thmref{thm: finite}.}}
Learner triplet queries to learn local Mahalanobis distance functions $\curlybracket{\sf{H}_c}_{c \in \cC}$ with error threshold $\xi$ as shown in \thmref{thm: smoothhessian}\; \tcp{\textcolor{bblue}{$\curlybracket{\sf{H}_c}_{c \in \cC}$ approximate the Hessian matrices $\curlybracket{\sH}_{c \in \cC}$ within the $\curlybracket{\tau_c}_{c \in \cC}$ factor as shown in \thmref{thm: advmahamain}.}}
Learner picks $\theta := 4\hat{\beta}$ as shown in \thmref{thm: smoothhessian}\;
For any $x \in \cX$, set $c(x) := \argmin_{c \in \mathcal{C}} \|x - c\|_2$\;
\ForEach{ $(x,y) \in \mathcal{X}^2$ to compute $d'(x, y)$}{
    \If{$(y - x)^\top \sf{H}_{c(x)} (y - x) > \theta$}{
         $d'(x, y) = d_{f}(c(x), c(y)) + \theta$\;
    }
    \Else{
         $d'(x, y) = (y - x)^\top \sf{H}_{c(x)} (y - x)$\;
    }
}
Output $d'$.
\end{algorithm}

In this Appendix, we provide the proof of \thmref{thm: smoothhessian}. It is divided broadly in three subsections: \appref{subapp: close} where we show desired local Mahalanobis approximation of a target Hessian matrix in \lemref{lem: taubound}, \appref{subapp: nice} where we achieve some well-behaved properties of the underlying distance function $d$ in results--\lemref{lem: convex} and \lemref{lem: smalldist}, and the final proof of \thmref{thm: smoothhessian} is given in \appref{subapp: final}.

We provide an extended version of \algoref{alg: smoothdisthessian} in \algoref{alg: smoothdisthessianfull}. 



\paragraph{Hessian continuity} In \assref{assump: a4}, we discussed the Hessian continuity of $d$. This assumption allows using the same Hessian matrix to compute distances in close neighborhood of a sample.

To see this, for a sample $s \in \cX$ assume that $d(x, \cdot)$ satisfies \assref{assump: a1}. Using the linear approximation in \lemref{lem: taylorlemma}, we note that 
\begin{align}
    d(x', y') &\approx \frac{1}{2}(x' - y')^{\top}\sf{H}_{x'}(x' - y') \label{eq: 11}\\
    d(x, x - x' + y') &\approx \frac{1}{2}(x' - y')^{\top}\sf{H}_{x}(x' - y') \label{eq: 12}
\end{align}
But we can bound the difference assuming (wlog) \eqref{eq: 11} $\ge$ \eqref{eq: 12},
\begin{align*}
    (x' - y')^{\top}(\sf{H}_{x'}- \sf{H}_{x})(x' - y') \le ||\sf{H}_x - \sf{H}_{x'}||_{F}\cdot||x' -y'||^2_2
\end{align*}
But using the Hessian continuity of $d$ we get 
\begin{align}
    (x' - y')^{\top}(\sf{H}_{x'}- \sf{H}_{x})(x' - y')\le L\cdot ||x - x'||_2\cdot ||x' -y'||^2_2 \label{eq: hessian}
\end{align}
Now, if $x, x', y'$ are in close proximity the difference in Mahalanobis distances is small. Thus, to approximate $d(x', y')$, we can use curvature information at $x$ itself, in the form of the Hessian matrix $\sf{H}_x$.

 \subsection{Close approximation of a Hessian matrix}\label{subapp: close}

 Using \algoref{alg:learnapproxmaha}, \algoref{alg: learnapproxmahamain1}, and \algoref{alg: learnapproxmahamain} we demonstrated how to learn the Mahalanobis distance function in both noiseless and noisy settings. While these approximations are generally tight, they may deviate by a factor, as stated in \thmref{thm: advmahamain}. In the following proof, we show how this factor influences the eigenvalues of the approximated Mahalanobis distance function. 
 
 The result is stated below, followed by the proof.
\begin{lemma}\label{lem: taubound} 
Let $\cX \subset \reals^p$ be a subset. Consider a distance function $d: \cX \times \cX \to \reals_{\ge 0}$ that satisfies \assref{assump: a1} and \assref{assump: a2} with scalars $\gm, \gM, M > 0$. For a given sample $c \in \cX$, assume $\sf{H}_c$ is an approximation of the Hessian matrix $\sH$ such that $\|\tau_c H_c^* - H_c\|_F \leq \xi$ for some scalar $\xi > 0$ (where $\tau_c$ is chosen as per \algoref{alg: learnapproxmahamain}). Then, we have the following bounds if $\xi \leq \min(1, \gm/2\gM)$: 
\begin{gather}
        \lambda_{\sf{min}}(\sf{H}_c) \ge \frac{\tau_c\gm}{2}, \quad \lambda_{\sf{max}}(\sf{H}_c) \le 2\tau_c\gM   \label{eq: eigenvalue}
    \end{gather}
\end{lemma}
\begin{proof} 
First, we note \( \tau_c = \frac{1}{y^\top \sH y} \), i.e., \( \tau_c \) is a scalar corresponding to the Mahalanobis distance function \( \sH \) that scales the quadratic product \( y^{\top} \sH y = 1 \) (see proofs of \thmref{thm: mahamain} and \thmref{thm: advmahamain} in \appref{app: maha}).

Given that $||\tau_c \cdot \sH - \sf{H}_c||_F \le \xi$ using Weyl's inequality we have the following bound on the eigenvalues of $\tau_c \sH$ and $\sf{H}_c$
 \begin{align*}
        \forall (i \in \bracket{p}),\,\,\, |\tau_c\lambda_i(\sH) - \lambda_i(\sf{H}_c)| \le \xi
     \end{align*} 
Thus we get the following desired bounds
\begin{align}\label{eq: eigenbound}
|\lambda_{\sf{min}}(\tau_c \sf{H}^*_c) - \lambda_{\sf{min}}(\sf{H}_c)| &\le \frac{\tau_c \gm }{2}, \quad \text{implying} \quad \lambda_{\sf{min}}(\sf{H}_c) \ge \frac{\tau_c \gm}{2}, \\
|\lambda_{\sf{max}}(\tau_c \sf{H}^*_c) - \lambda_{\sf{max}}(\sf{H}_c)| &\le \frac{\tau_c \gm }{2}, \quad \text{implying} \quad 2 \tau_c \gM \ge \lambda_{\sf{max}}(\sf{H}_c).
\end{align}
This proves the claim in \eqnref{eq: eigenvalue} of the lemma.
\end{proof}
\subsection{Well-behaved properties of a smooth distance function}\label{subapp: nice}

In this subsection, we provide several key properties of a smooth distance function that satisfies the conditions in \thmref{thm: smoothhessian}. The first result, as stated in \lemref{lem: convex}, establishes the strong convexity (see \defref{def:strongconv}) and smoothness (see \defref{def:smooth}) of 
$d$ in a small neighborhood of a given sample, ensuring predictable behavior and enabling us to measure the rate of growth of the function. The second result, as stated in \lemref{lem: smalldist}, demonstrates how local Mahalanobis approximations, using \algoref{alg: learnapproxmahamain}, produce lower bounds on distance values within a small neighborhood. To visually understand the proofs, we illustrate them in \figref{fig:local}.

We begin by introducing the concepts of strong convexity and smoothness for a differentiable function. These notions allow us to study the local properties of a distance function \(d\) that satisfies \assref{assump: a1}.

\begin{definition}[Pointwise \(\mu\)-strong convexity]\label{def:strongconv}
Let \(f: \mathcal{S} \to \mathbb{R}\) be a differentiable function, where \(\mathcal{S} \subset \mathbb{R}^p\). We say \(f\) is \(\mu\)-strongly convex at \(x_0 \in \mathcal{S}\) if
\[
    f(x) - f(x_0) \geq \nabla f(x_0) \cdot (x - x_0) + \mu\cdot \|x - x_0\|_2^2, \quad \forall x \in \mathcal{S}.
\]
\end{definition}

\begin{definition}[Pointwise \(\gamma\)-smoothness]\label{def:smooth}
Let \(f: \mathcal{S} \to \mathbb{R}\) be a differentiable function, where \(\mathcal{S} \subset \mathbb{R}^p\). We say \(f\) is \(\gamma\)-smooth at \(x_0 \in \mathcal{S}\) if
\[
    f(x) - f(x_0) \leq \nabla f(x_0) \cdot (x - x_0) + \gamma\cdot \|x - x_0\|_2^2, \quad \forall x \in \mathcal{S}.
\]
\end{definition}

We can apply these concepts to examine the local behavior of a distance function \(d\) that satisfies \assref{assump: a1}. Fixing a sample \(x \in \cX\), we use Taylor's approximation from \eqnref{eq: 1} to rewrite $d(x,x+h)$ where $h$ is within a small neighborhood of $x$ wrt $\ell_2$ norm as
\[
    d(x, x + h) = \frac{1}{2} h^\top \sHx h + R_{x,2}(h),
\]
where \(d(x, x) = 0\) and \(\partial_x d(x, \cdot) = 0\). To bound \(R_{x,2}(h)\), we rely on \lemref{lem: taylorlemma}.

Next, we present \lemref{lem: convex}, with the proof outlined in the subsequent discussion.
\begin{lemma}\label{lem: convex}
    Let $\cX \subset \reals^p$ be a subset. Consider a distance function $d: \cX \times \cX \to \reals_{\ge 0}$ that satisfies \assref{assump: a1} and \assref{assump: a2} with scalars $\gm, \gM, M > 0$. Then for any sample $x \in \cX$ the function $d(x,\cdot)$ is $\frac{\gm}{4}$-strongly convex and $\gM$-smooth in a neighborhood of $\frac{3\gm}{2M p^{\frac{3}{2}}}$.
\end{lemma}
\begin{proof}
    We proof the claim of the statement using the Taylor's approximation as shown in \lemref{lem: taylorlemma}.
    
    First, note that for any fixed sample $x \in \cX$, and $h$ in a small neighborhood of $x$ we have the following
    \begin{align*}
        d(x, x+ h) \ge \frac{1}{2}h^{\top}\sHx h - \frac{M p^{\frac{3}{2}}}{6} \cdot||h||^{3}_2
    \end{align*}
    Now, if $||h||_2$ is small enough, i.e. smaller than $\frac{3\gm}{2M p^{\frac{3}{2}}}$, then we have 
    \begin{align}
        d(x, x+ h) \ge \frac{\gm}{2}||h||^2_2 - \frac{\gm}{4}||h||^2_2 = \frac{\gm}{4}||h||^2_2 \label{eq: strongconvex}
    \end{align}
    Thus, $d(x,\cdot)$ is $\frac{\gm}{4}$-strongly convex at $x$. Since $\gm$ is the smallest eigenvalue for the Hessian at any arbitrary sample $x \in \cX$, this holds over $\cX$. On the other hand, the function $d(x,\cdot)$ is smooth, i.e.,
    \begin{align}
        d(x, x+ h) \le \frac{1}{2}h^{\top}\sHx h + \frac{M p^{\frac{3}{2}}}{6} \cdot||h||^{3}_2 \le \frac{\gM}{2}||h||^2_2 +  \frac{\gm}{4}||h||^2_2 \le \gM ||h||^2_2 \label{eq: smoothconvex}
    \end{align}
    This completes the proof of the statement.
\end{proof}

In the previous result, we studied the strong convexity and smoothness of $d$. We now demonstrate how local approximations via the Mahalanobis distance function help in bounding 
$d$ within small neighborhoods. This result is formalized in \lemref{lem: smalldist}, with the proof provided thereafter.

    \begin{lemma}\label{lem: smalldist}
        Let $\cX \subset \reals^p$ be a subset. Consider a distance function $d: \cX \times \cX \to \reals_{\ge 0}$ as stated in \thmref{thm: smoothhessian}. Fix $\delta \le \frac{3\gm}{2M p^{\frac{3}{2}}}$, and define
        \begin{align*}
             \Delta_{\delta} := \min_{x \in \cX} \curlybracket{d(x,x') : x' \in \cX \setminus B_2(x,\delta)}
        \end{align*}
     For a given sample $c \in \cX$, assume $\sf{H}_c$ is an approximation of the Hessian matrix $\sH$ such that $\|\tau_c H_c^* - H_c\|_F \le \xi$, where $\xi \leq \min(1, \gm/2\gM)$. Then, for any fixed sample $x_0 \in \cX$ and a scalar $\hat{\beta} \le \frac{4\Delta_\delta \gM}{\gm^2\g}$ where $\g \ge 1$, the following property holds on $\sf{H}_c$:
    \begin{align*}
        \forall y \in \cX,\,\, t \in \bracket{0, \g}, \quad (x_0 - y)^{\top}\sf{H}_c(x_0 - y) \ge 2t\hat{\beta} \quad \text{implying} \quad d(x_0,y) \ge \frac{\gm^2\cdot t \hat{\beta}}{4\gM}.
    \end{align*}
    \end{lemma}
\begin{figure}[t!]
    \centering
    \includegraphics[width=0.65\textwidth]{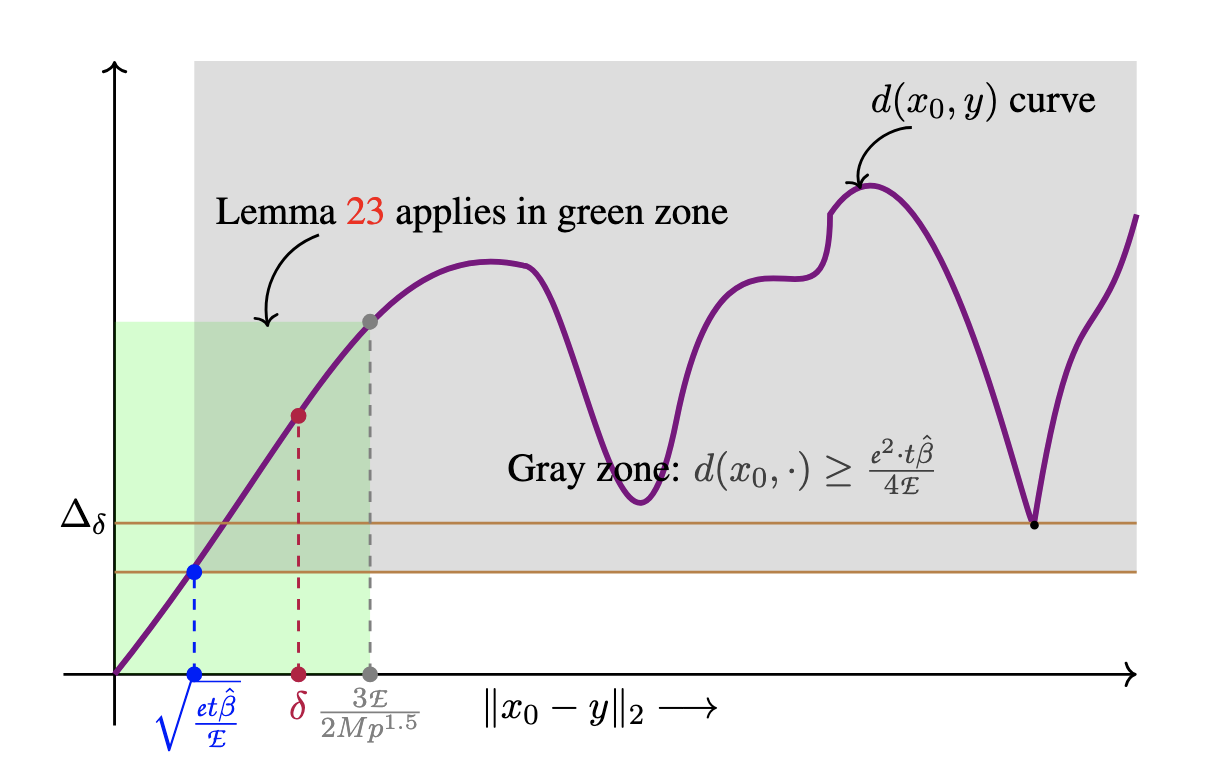} 
    \caption{
    The illustration is based on the results of \lemref{lem: convex} and \lemref{lem: smalldist}. The violat curve represents the function \( d(x_0, \cdot): \cX \to \reals_{\geq 0} \). The scalar \( \Delta_\delta := \min_{x \in \cX} \{ d(x, x') : x' \in \cX \setminus B_2(x, \delta) \} \), with \textcolor{purple}{\( \delta = \frac{4\gm^2}{M p^{1.5} \gM} \)}, shows the minimum value of \( d(x_0, \cdot) \) regardless of the choice of \( x_0\). The green zone, located to the left of the scalar \textcolor{gray}{\( \frac{3 \gM}{2M p^{1.5}} \)}, represents the interval where \( d(x_0, \cdot) \) is both strongly convex and smooth, as demonstrated in \lemref{lem: convex}.
Assume that the scalar \( \hat{\beta} \) is selected so that the distance \textcolor{blue}{\( \sqrt{\frac{\gm t\hat{\beta}}{\gM}} \)}, for any $t \in \bracket{0, \g}$, is within the safe green zone. Then the gray zone, to the right of \textcolor{blue}{\( \sqrt{\frac{\gm t\hat{\beta}}{\gM}} \)}, indicates that \( d(x_0, \cdot) \) is lower-bounded by \textcolor{red}{\( \frac{\gm^2 \cdot t\hat{\beta}}{4\gM} \)}, as shown in \lemref{lem: smalldist}.  }
    \label{fig:local}
\end{figure}

    \begin{proof} 
Pick any $t \in \bracket{0, \g}$ and $y \in \cX$. Then
\begin{align*}
    (x_0 - y)^{\top}\sf{H}_c(x_0 - y) \ge 2 t\hat{\beta} \implies \ \|x_0 - y\|_2^2 \ge \frac{\gm t\hat{\beta}}{\gM}
\end{align*}
where we use the bounds from \lemref{lem: taubound}, i.e. for \( \tau_c \in \left[\frac{1}{\gM}, \frac{1}{\gm}\right] \), we have \( \lambda_{\sf{min}}(\sf{H}_c) \ge \frac{\tau_c \gm}{2} \) and \( 2\tau_c \gM \ge \lambda_{\sf{max}}(\sf{H}_c) \).

There are two possibilities: either $\|x_0 - y\| \leq \delta$ or $\|x_0 - y\| > \delta$.
If $\|x_0 - y\| \leq \delta$, note that we can use \lemref{lem: convex} ($\delta \le \frac{3 \gm}{2M p^{3/2}}$), and thus we obtain 
\begin{align*}
    d(x_0,y) \ge \frac{\gm^2t' \hat{\beta}}{4\gM} \ge \frac{\gm^2t \hat{\beta}}{4\gM}
\end{align*}
where we write $\|x_0 - y\| = \frac{\gm t' \hat{\beta}}{\gM}$ for some scalar $t'$.

On the other hand, if $\|x_0 - y\| > \delta$, by the definition of $\Delta_\delta$, we have 
\begin{align*}
    d(x_0, y) \ge \Delta_\delta > \frac{\gm^2 \cdot\g \hat{\beta}}{4\gM}
\end{align*}
Thus, we have shown that
\[
\forall y \in \cX\,\, t \in \bracket{0, \g}, \text{ s.t. } (x_0 - y)^{\top}\sf{H}_c(x_0 - y) \ge 2 t\hat{\beta} \implies  d(x_0, y) \ge \frac{\gm^2 \cdot t\hat{\beta}}{4 \gM}.
\]
This completes the proof of the lemma.
\end{proof}

With this we provide the proof of the main theorem of \secref{sec: mult}.

   \subsection{Proof of \thmref{thm: smoothhessian}}\label{subapp: final}

    In the proof below, we denote the multiplicative factor as \( \re \) to distinguish it from other notations. The learner selects a cover of size \( \epsilon > 0 \) for global approximation and an error threshold \( \xi \) for learning local Mahalanobis distances. We show that \algoref{alg: smoothdisthessianfull}, when operated with \( \epsilon \) and \( \xi \) bounded by certain functions of \( \re \), dependent on the smoothness parameters of \( d \), achieves triplet equivalence up to a multiplicative factor of \( \re \).

   The learner queries the user in the following manner:

   \begin{itemize}
       \item[a)] \hypertarget{item1}{(\textcolor{black}{\tt{Global strategy}}) The learner queries to learn a finite sample distance function $d_f$ on centers $\cC$ corresponding to a  $\epsilon$-cover of the space wrt to $\ell_2$ distance up to triplet equivalence
       using \algoref{alg: finite} where 
       \begin{gather}
           \epsilon \le \sqrt{\frac{\gm^2 \hat{\beta}\re}{16\gM^2(1 + \re)}}, \label{eq: g}\text{where }\\
           \hat{\beta} := \min \Bigg\{ 
\underbrace{{\frac{\gm}{8\gM}\frac{9\gm^2}{4M^2p^3}},{ \frac{4\Delta_\delta \gM}{\gm^2\g}}}_{\text{\shortstack{Based on \\\lemref{lem: convex}, \lemref{lem: smalldist}}}}, 
\left( \frac{\gm \re}{8 (\re + 2)\left( \frac{M p^{1.5}}{6}\sqrt{\frac{8\g\gM}{\gm}} + \frac{L}{2}\sqrt{\re} \right)} \right)^2 \Bigg\},\label{eq: hat}
       \end{gather}
       where we set $\g \ge \frac{40 \gM^3}{\gm^3}$, and $\Delta_{\delta} := \min_{x \in \cX} \curlybracket{d(x,x') : x' \in \cX \setminus B_2\paren{x,\delta}}$ for $\delta \le \frac{3\gm}{2M p^{\frac{3}{2}}}$. 
       For this choice of the cover, using the claim in \thmref{thm: finite}, the query complexity is at most $\cN(\cX,\epsilon,\ell_2)^2 \log \cN(\cX,\epsilon,\ell_2)$ to learn $d_f$.
       }  
       \newline
       \item[b)] \hypertarget{item2}{(\textcolor{black}{\tt{Local strategy}}) 
        On each center \( c \in \cC \), using \algoref{alg: learnapproxmahamain}, the learner queries triplet comparisons to learn the local Mahalanobis distance functions \( \{ \sf{H}_c \}_{c \in \cC} \), with an error threshold\footnote{As per \remref{rem: errorbound}, we can avoid the bound of $\xi \leq 3 \gm^3/(2Mp^{3/2} \gM^2)$ as stated in \thmref{thm: advmahamain} assuming that the learner can estimate the scaling factor $\rho$ (see discussion preceding \thmref{thm: advmahamain}).} \( \xi \leq \frac{\gm \re}{4\gM (\re + 2) } \).}
   The query complexity for learning the local Mahalanobis distance functions is \( \cN(\cX,\epsilon, \ell_2) \cdot (p^2 \log \paren{\frac{8p^2\gM^3(\re + 2)}{\gm^3\re}} + p) \), using \thmref{thm: advmahamain}. 
   \end{itemize}

    
   
   With this, we can define the approximated distance function $d'$ as follows: for all $x,y \in \cX$
    \begin{equation}\label{eq: strategy}
        d'(x,y) = \begin{cases}
            d_{f}(c(x), c(y)) + \theta & \textnormal{ if } (y - x)^{\top}\sf{H}_{c(x)}(y - x) > \theta\\
             (y - x)^{\top}\sf{H}_{c(x)}(y - x) & \textnormal{ o.w. }
        \end{cases}
    \end{equation}
    where $\theta = 4\hat{\beta}$ is the threshold to switch between the global distance function $d_f$ and local Mahalanobis distance function. 
    
    Based on this distance function we can design a triplet prediction strategy. First consider the following definitions for any triplet $x,y,z \in \cX$ as 
\[
\ell_{ab} := (b - a)^\top \mathsf{H}_{c(x)} (b - a), \quad \text{for } a, b \in \{x, y, z\}.
\]
Given this, the case-by-case conditions for triplet classification are:
\[
\begin{cases}
    \text{If } \ell_{xy} > \theta \text{ and } \ell_{xz} > \theta, & \text{Answer according to the triplet } (c(x), c(y), c(z)) \text{ using } d_f. \\
    \text{If } \ell_{xy} \le \theta \text{ and } \ell_{xz} \le \theta, & \text{Answer according to } \sgn{\ell_{xy} - \ell_{xz}}. \\
    \text{If } \ell_{xy} > \theta \text{ and } \ell_{xz} \le \theta & \text{Answer } ``d(x, y) > d(x, z)".\\
    \text{Otherwise,} & \text{Answer } ``d(x, z) > d(x, y)". \\
\end{cases}
\]
    
    In the rest of the proof, we provide the correctness of the query strategy for the approximated distance function $d'$ as shown in \eqnref{eq: strategy}.

    \noindent\textbf{Case 1:} The distances are far apart: both \( (y - x)^{\top}\sf{H}_{c(x)}(y - x) \) and \( (z - x)^{\top}\sf{H}_{c(x)}(z - x) \) are greater than \( 2 \hat{\beta} \) (note the factor that shows that the global approximation is still fine within this factor). In this case, the learner finds the centers \( (c(x), c(y), c(z)) \) for the samples and uses the finite sample distance function \( (\cC, d_f) \) to answer the query.

    We show that if \( d(x,y) > (1 + \re) d(x,z) \), then \( d(c(x), c(y)) > d(c(x), c(z)) \), which in turn implies \( d_f(x,y) > d_f(x,z) \), as \( d_f \) is triplet-equivalent to \( d \) on \( \cd \) (see \thmref{thm: finite}).

Note that $\epsilon \le \sqrt{\hat{\beta}} \le \frac{3\gm}{2Mp^{1.5}}$. Thus, if \( ||c(x) - x||_2 \le \epsilon \), the smoothness of \( d \) in the second argument, as given by \eqnref{eq: smoothconvex} (\lemref{lem: convex}), bounds \( d(c(x), x) \) and \( d(x, c(x)) \) as follows:
\begin{subequations}\label{eq: closebound}
    \begin{align}
        d(c(x), x) &\le \gM \cdot ||c(x) - x||^2_2 \le \gM \cdot \epsilon^2, \\
        d(x, c(x)) &\le \gM \cdot ||c(x) - x||^2_2 \le \gM \cdot \epsilon^2.
    \end{align}
\end{subequations}
Note that this holds for any \( x \in \cX \). 

    
    Now, we will show how reducing $x,y,z$ to the closest centers in $\cC$ leads to correct classification:
    \allowdisplaybreaks
    \begin{align}
        d(c(x), c(y)) &\ge d(c(x), y) - d(c(y), y) \label{eq: glob0}\\
                          & \ge d(x,y) - d(x, c(x)) - d(c(y), y)\label{eq: glob1}\\
                          & > (1+ \re)\cdot d(x,z) - 2\gM\cdot \epsilon^2\label{eq: glob2}\\
                          & \ge \paren{d(c(x),c(z)) - d(x,c(x)) - d(c(z),z)} + \re d(x,z) - 2\gM\cdot \epsilon^2\label{eq: glob3}\\
                          & \ge d(c(x),c(z)) + \re d(x,z) - 4 \gM\cdot \epsilon^2\label{eq: glob4}\\
                           & \ge d(c(x),c(z)) + \paren{\frac{\gm^2 \hat{\beta} \re}{4\gM}} - 4 \gM\cdot  \paren{\frac{\gm^2 \hat{\beta}\re}{16\gM^2(1 + \re)}}\label{eq: glob5}\\
                           & \ge d(c(x),c(z)) + \paren{\frac{\gm^2 \hat{\beta} \re}{4\gM}} - \paren{\frac{\gm^2 \hat{\beta} \re}{4\gM(1 + \re)}}\label{eq: glob6}\\
                          & > d(c(x),c(z)) \label{eq: glob9}
    \end{align}
    In the first two equations: \eqnref{eq: glob0}-\eqref{eq: glob1} we have used the triangle inequality for $d$. In \eqnref{eq: glob2}, we note that the conditions of \lemref{lem: convex} are satisfied due to the choice of the cover radius $\epsilon$ and thus the bound follows using \eqnref{eq: closebound}. In \eqnref{eq: glob3}, we expand $d(x,z)$ using triangle inequality. In \eqnref{eq: glob4}, we use the bounds from \eqnref{eq: closebound} on $d(x,c(x))$, and $d(c(z),z)$.
    In \eqnref{eq: glob5}, we note that the conditions for \lemref{lem: smalldist} are satisfied and thus $d(x,z) \ge \paren{\frac{\gm^2 \hat{\beta}}{4\gM}}$ for the choice of $t=1$ is implied as \( (z - x)^{\top}\sf{H}_{c(x)}(z - x) \ge 2\hat{\beta} \). Furthermore, we expand $4\gM\cdot \epsilon^2$ to complete the final bound in \eqnref{eq: glob9}. Hence, in the \textbf{Case 1} we have shown the multiplicative approximation with $d'$. 
    
    Now, we will consider the second case.
    
    \noindent\textbf{Case 2:} The distances are small: $(y - x)^{\top}\sf{H}_{c(x)}(y - x)$ and $(z - x)^{\top}\sf{H}_{c(x)}(z - x)$ are smaller than $4 \g\hat{\beta}$ (note the choice of the upper bound which shows that the local approximation is still fine within a larger factor).
    
     Under this condition, learner decides the triplet comparison on $(x,y,z)$ by checking the label in $\sgn{(y - x)^{\top}\sf{H}_{c(x)}(y - x) - (z - x)^{\top}\sf{H}_{c(x)}(z - x)}$. We show that it is positive if $d(x,y) > (1 + \re)d(x,z)$.

    In the rest of the proof, we denote $\bigstar := (y - x)^{\top}\sf{H}_{c(x)}(y - x) - (z - x)^{\top}\sf{H}_{c(x)}(z - x)$.
    
    The key idea is to lower bound the difference in the distances computed by the learner, i.e. $\bigstar$ in terms of $d(x,y)$ and $d(x,z)$ and show that the approximation error remains within control.
    For the choice of $\xi$ as stated earlier, using \thmref{thm: advmahamain}, we can bound difference of evaluations at $(x - y)$ as follows:
    \begin{align}
    (x - y)^{\top} \left( \tau_{c(x)} \cdot \sHc{x} - \sf{H}_{c(x)} \right) (x - y) \le ||\tau_{c(x)} \cdot \sHc{x} - \sf{H}_{c(x)}||_F \cdot ||x - y||_2^2 
    \le \xi\cdot ||x - y||_2^2. \label{eq: quad}
    \end{align}
    where \( \tau_{c(x)}= \frac{1}{z^\top \sHc{x} z} \) ($z$ is chosen such that largest diagonal element of $\tau_{c(x)} \sHc{x}$ is 1, see proofs of \thmref{thm: mahamain} and \thmref{thm: advmahamain} in \appref{app: maha}).
    
    Similarly, using \eqnref{eq: quad}, we can show a lower bound on $(z - x)^{\top}\sf{H}^*_{c(x)}(z - x)$
    \begin{align}
        (z - x)^{\top}\sf{H}^*_{c(x)}(z - x) \ge \frac{1}{\tau_{c(x)}} \bracket{(x - z)^{\top}\sf{H}_{c(x)}(x - z)  - \xi\cdot ||x -z||^2_2} \label{eq: lb1}
    \end{align} 
    Using \eqnref{eq: hessian} due to Hessian continuity of $d$ we have 
    \begin{subequations}\label{eq: bound1}
    \begin{align}
        \frac{1}{2}(y - x)^{\top}\sf{H}^*_{c(x)}(y - x) &\ge \frac{1}{2}(y - x)^{\top}\sf{H}^*_{x}(y - x) - \frac{L}{2}\cdot ||x - c(x)||_2\cdot ||x - y||_2^2\\
        &\ge d(x,y) - \frac{M p^{\frac{3}{2}}}{6}||x - y||_2^3 - \frac{L}{2}\cdot ||x - c(x)||_2\cdot ||x - y||_2^2\\
        & = d(x,y) -\underbrace{\paren{\frac{M p^{\frac{3}{2}}}{6}||x - y||_2 + \frac{L}{2}\cdot ||x - c(x)||_2}}_{\alpha(x,y)} ||x - y||_2^2 \label{eq: mult1}
    \end{align}
    \end{subequations}
    In the second equation, we have used the Taylor's approximation using \lemref{lem: taylorlemma}. With a similar manipulation, we can achieve the following lower bound 
    \begin{align}
        d(x,z) \ge \frac{1}{2}(z - x)^{\top}\sf{H}^*_{c(x)}(z - x) - \underbrace{\paren{\frac{M p^{\frac{3}{2}}}{6}||x - z||_2 + \frac{L}{2}\cdot ||x - c(x)||_2}}_{\alpha(x,z)} ||x - z||_2^2 \label{eq: mult2}
    \end{align}
    With the approximations above, we can lower bound $\bigstar$ as
    \begin{align}
        \bigstar &\ge \tau_{c(x)}\paren{ (y - x)^{\top}\sf{H}^*_{c(x)}(y - x) - (z - x)^{\top}\sf{H}^*_{c(x)}(z - x)} - \xi \cdot ||x-y||_2^2 - \xi \cdot ||x-z||_2^2 \label{eq: low1}\\
        &\ge 2\tau_{c(x)}\paren{ d(x,y) - \alpha(x,y)\cdot ||x-y||_2^2 - d(x,z) - \alpha(x,z)\cdot ||x-z||_2^2} - \xi \cdot ||x-y||_2^2 - \xi \cdot ||x-z||_2^2  \label{eq: low2}\\
        & \ge 2\tau_{c(x)}\underbrace{\paren{d(x,y)  - d(x,z)}}_{\dagger} - (2\tau_{c(x)} \alpha(x,y) + \xi) \cdot ||x-y||_2^2  - (2\tau_{c(x)}\alpha(x,z) + \xi)\cdot ||x-z||_2^2  \label{eq: low3}
    \end{align}
    \eqnref{eq: low1} follows as a direct consequence of \eqnref{eq: quad} and \eqnref{eq: lb1}. In \eqnref{eq: low2} we use the bounds on the distances $d(x,y)$ and $d(x,z)$ as shown in \eqnref{eq: bound1}-\eqnref{eq: mult2}. 

    Now, we can bound $\dagger$ as
    \begin{align*}
        d(x,y) - d(x,z) \ge \paren{\frac{\re}{\re + 2}} \paren{d(x,y) + d(x,z)}
    \end{align*}
    where we have used the inequality: if $A > (1+\re)B$ then $A-B > \frac{\re}{2+\re} \cdot (A + B)$. Thus, we can rewrite \eqnref{eq: low3} as 
    \begin{align}
        \bigstar \ge  \paren{\frac{2\tau_{c(x)}\re }{(\re + 2)}} \paren{d(x,y) + d(x,z)} - (2\tau_{c(x)} \alpha(x,y) + \xi) \cdot ||x-y||_2^2  - (2\tau_{c(x)}\alpha(x,z) + \xi)\cdot ||x-z||_2^2 \label{eq: low4}
    \end{align}
    To complete the proof of the \textbf{Case 2} we will show the following bounds 
    \begin{align*}
        & d(x,y) > \frac{(\re + 2)}{\re } \paren{\alpha(x,y) + \frac{\xi}{2\tau_{c(x)}}}\cdot ||x-y||_2^2\\
        & d(x,z) \ge \frac{(\re + 2)}{\re } \paren{\alpha(x,z) + \frac{\xi}{2\tau_{c(x)}}} \cdot ||x-z||_2^2 
    \end{align*}
    We will show that the constants $\alpha(x,y)$ and $\alpha(x,z)$ are upper bounded by $\frac{\gm\re}{8(\re + 2)}$. 
    Note that given the condition that $(y - x)^{\top}\sf{H}_{c(x)}(y - x) \le 4\g\hat{\beta}$ and the choice of the cover radius $\epsilon$ we have the following bounds on the distances $x,y$ and the center $c(x)$ 
    \begin{align*}
        ||x - y||^2_2 \le \frac{8\g\hat{\beta}}{\tau_{c(x)}\gm} \le \frac{8\g\gM\hat{\beta}}{\gm},\quad ||x - c(x)||^2_2 \le \hat{\beta}\re
    \end{align*}
    where in the first bound we have used \lemref{lem: taubound} (note that $\xi$ satisfies the stated bound) and in the second the definition of $\epsilon$.

    Now, using the bound in \eqnref{eq: hat}) we have 
    \begin{align*}
       \alpha(x,y) =  \paren{\frac{M p^{\frac{3}{2}}}{6}||x - y||_2 + \frac{L}{2}\cdot ||x - c(x)||_2} \le \paren{\frac{M p^{\frac{3}{2}}}{6}\sqrt{\frac{8\g\gM}{\gm}} + \frac{L}{2}\sqrt{\re}}\cdot \sqrt{\hat{\beta}} < \frac{\re \gm}{8(\re + 2)}
    \end{align*}
    This gives the stated bound on $\alpha(x,y)$ and $\alpha(x,z)$. Thus, we have 
    \begin{align*}
        \frac{\re + 2}{\re } \paren{\alpha(x,y) + \frac{\xi}{2\tau_{c(x)}}} < \frac{\gm}{4}, \quad
        \frac{\re + 2}{\re } \paren{\alpha(x,z) + \frac{\xi}{2\tau_{c(x)}}} < \frac{\gm}{4}
    \end{align*}

Using Taylor's approximation in \lemref{lem: taylorlemma} and noting that $x, y$ and $z$ satisfy the condition--$||x-y||_2, ||x-z||_2 \le \frac{3\gm}{2Mp^{1.5}}$, we have 
\begin{align*}
    &d(x,y) \ge \frac{1}{4}(x-y)^\top\sf{H}^*_{x}(x-y) \ge   
    \frac{\gm}{4} ||x-y||_2^2 > \frac{(\re + 2)}{\re } \paren{\alpha(x,y) + \frac{\xi}{2\tau_{c(x)}}}\cdot ||x-y||_2^2\\
    &d(x,z) \ge \frac{1}{4}(x-z)^\top\sf{H}^*_{x}(x-z) \ge \frac{\gm}{4} ||x-z||_2^2 \ge \frac{(\re + 2)}{\re } \paren{\alpha(x,z) + \frac{\xi}{2\tau_{c(x)}}}\cdot ||x-z||_2^2 
\end{align*}
where in the first equation we get strictly inequality because $x \neq y$ (otherwise $d(x,y) \not > (1 + \re ) d(x,z)$).

But this implies that in \eqnref{eq: low4} we achieve a strict positive inequality, i.e.
\begin{align*}
    \bigstar > 0.
\end{align*}
This completes the proof of the second case.

    Now, in the final case, we have the following possibility.
     

    \noindent\textbf{Case 3:} The distances are small: one of $(y - x)^{\top}\sf{H}_{c(x)}(y - x)$ and $(z - x)^{\top}\sf{H}_{c(x)}(z - x)$ is smaller\footnote{not necessarily strictly} than $4 \hat{\beta}$ and the other \tt{strictly} greater than $4 \hat{\beta}$.
    For ease of notation, we use $\ell_{xy} =(y - x)^{\top}\sf{H}_{c(x)}(y - x)$ and $\ell_{xz} = (z - x)^{\top}\sf{H}_{c(x)}(z - x)$.   
    We analyze two subcases according to the values of $\ell_{xy}$ and $\ell_{xz}$.

First, if \( \ell_{xy}, \ell_{xz} \le 4\g\hat{\beta} \), then using the analysis in \textbf{Case 2}, we note that
\[
d(x,y) > (1 + \re)d(x,z) \implies \ell_{xy} > \ell_{xz} \implies \ell_{xy} > 4\hat{\beta}, \ell_{xz} \le 4\hat{\beta} \implies d'(x,y) > d'(x,z).
\]

Second, if one of \( \ell_{xy} \) or \( \ell_{xz} \) is greater than \( 4\g \hat{\beta} \), assume without loss of generality that \( \ell_{xy} > 4\g \hat{\beta} \). Then, using \lemref{lem: smalldist} and setting \( t = \g \), we have
\[
d(x,y) \ge \frac{\gm^2 \cdot \g \hat{\beta}}{4\gM} = \frac{10\gM^2\hat{\beta}}{\gm}.
\]
Since \( \ell_{xz} < 4\hat{\beta} \) we have 
\[(x - z)^{\top}\sf{H}_{c(x)}(x-z) < 4 \hat{\beta} \implies ||x-z||_2^2 < \frac{8 \gM \hat{\beta}}{\gm},\]
where we use the bound \( 2\tau_{c(x)} \gM \ge \lambda_{\sf{max}}(\sf{H}_{c(x)}) \) using \lemref{lem: taubound}.

Since \( \sqrt{\frac{8 \gM \hat{\beta}}{\gm}} \) is smaller than \( \frac{3 \gm}{2M p^{3/2}} \), we can invoke \lemref{lem: convex} to apply the smoothness property of \( d \), which yields:
\begin{align*}    
d(x,z) \le \frac{8 \gM^2 \hat{\beta}}{\gm}.                                                                                                                                                                                                                                                  
\end{align*}

Thus, \( d(x,y) > d(x,z) \), enforcing the condition \( d(x,z) \not > (1 + \re) d(x,y) \). Consequently,
\[
d(x,y) > (1 + \re) d(x,z) \implies \ell_{xy} > \ell_{xz} \implies d'(x,y) > d'(x,z)
\]
is satisfied in this subcase as well.

    Thus, we have demonstrated the correctness of the query strategy based on the approximated distance function \( d' \) as constructed in \eqnref{eq: strategy}. This establishes that \algoref{alg: smoothdisthessianfull} outputs a distance function that is \( \re \)-multiplicatively approximate to the underlying distance function \( d \).

   
}
{}
\end{document}